\documentclass[11pt]{article}

\let\counterwithin\relax
\usepackage{chngcntr}

\usepackage[a4paper, left=3cm, right=3cm, top=3cm, bottom=3cm]{geometry}
\usepackage{xr}
\RequirePackage{amsthm, amsmath, amsfonts, amssymb}%
\usepackage{mathrsfs}%
\usepackage{amsmath,mathtools}%
\usepackage{amssymb}%
\usepackage[titletoc,title]{appendix}%
\usepackage[utf8]{inputenc}%
\usepackage{dsfont}%
\usepackage{caption}
\usepackage{subcaption}
\usepackage{float}%
\usepackage{mathrsfs}%
\usepackage{graphicx, color}%
\usepackage{bm}
\usepackage{tcolorbox}

\usepackage{etoolbox}
\patchcmd{\appendices}{\quad}{. }{}{}

\usepackage{pgfplots}%
\usepackage{booktabs,multirow,array,multicol}
\newcommand{\otoprule}{\midrule[\heavyrulewidth]}

\usepackage[linesnumbered,ruled,vlined]{algorithm2e}

\usepackage{natbib}
\usepackage[titletoc,title]{appendix}%

\definecolor{darkblue}{rgb}{0.0,0.0,0.7}
\RequirePackage[%
colorlinks = true,%
linkcolor = darkblue,%
citecolor = darkblue,%
urlcolor = darkblue, %
]{hyperref}%
\hypersetup{%
	pdfauthor = {Mokhtar Z. Alaya},%
	pdftitle = {},%
	pdfcreator = {pdflatex},%
	pdfproducer = {pdflatex}}

\newtheorem{theorem}{Theorem}
\newtheorem{corollary}{Corollary}
\newtheorem{lemma}{Lemma}
\newtheorem{proposition}{Proposition}
\newtheorem{definition}{Definition}
\newtheorem{remark}{Remark}

\newtheorem{assumption}{Assumption}{\bf}{\rm}%

\usepackage{chngcntr}
\usepackage{apptools}
\AtAppendix{\counterwithin{lemma}{section}}
\AtAppendix{\counterwithin{theorem}{section}}
\AtAppendix{\counterwithin{definition}{section}}

\newcommand{\bigO}{\mathcal{O}}

\usepackage{datenumber}

\newcommand{\inr}[1]{\langle #1 \rangle}

\newcommand{\ind}[1]{{\mathds{1}}_{{#1}}}%
\newcommand{\norm}[1]{\|#1\|}
\newcommand{\R}{{\mathbb{R}}}
\newcommand{\E}{\mathds{E}}
\newcommand{\V}{\mathds{V}}

\newcommand{\bA}{{\boldsymbol A}}%
\newcommand{\bcA}{\boldsymbol{\mathcal{A}}}%
\newcommand{\bB}{{\boldsymbol B}}%
\newcommand{\bcB}{\boldsymbol{\mathcal{B}}}%
\newcommand{\bL}{{\boldsymbol L}}%
\newcommand{\bM}{{\boldsymbol M}}%
\newcommand{\bN}{{\boldsymbol N}}%
\newcommand{\bR}{\boldsymbol R}%
\newcommand{\bcQ}{\boldsymbol{\mathcal{Q}}}%
\newcommand{\bW}{\boldsymbol W}%
\newcommand{\bX}{\boldsymbol  X}%
\newcommand{\bcX}{\boldsymbol{\mathcal{X}}}%
\newcommand{\bcY}{\boldsymbol{\mathcal{Y}}}%
\newcommand{\bcW}{\boldsymbol{\mathcal{W}}}%
\newcommand{\bcM}{\boldsymbol{\mathcal{M}}}%
\newcommand{\bcZ}{\boldsymbol{\mathcal{Z}}}%
\newcommand{\bcG}{\boldsymbol{\mathcal{G}}}%
\newcommand{\bcU}{\boldsymbol{\mathcal{U}}}%
\newcommand{\bcV}{\boldsymbol{\mathcal{V}}}%
\newcommand{\bcR}{\boldsymbol{\mathcal{R}}}%
\newcommand{\bcI}{\boldsymbol{\mathcal{I}}}%
\newcommand{\bY}{\boldsymbol{Y}}%
\newcommand{\bZ}{\boldsymbol Z}
\newcommand{\bSigma}{\boldsymbol \Sigma}%

\newcommand{\bXi}{\boldsymbol{\Xi}}%
\renewcommand{\P}{\mathds{P}}
\newcommand{\bcdot}{\raisebox{-0.80ex}{\scalebox{1.8}{$\cdot$}}}
\newcommand{\cst}{\raisebox{-0.15ex}{\scalebox{1.30}{$c$}}}

\newcommand{\varsig}{\raisebox{-0.15ex}{\scalebox{1.30}{$\varsigma$}}}
\usepackage{accents}
\newcommand*{\dt}[1]{\accentset{\raisebox{0ex}{$\,\star$}}{#1}}
\newcommand*{\ddt}[1]{\accentset{\raisebox{0ex}{$\,\,\star$}}{#1}}

\DeclareMathOperator*{\argmin}{\arg\!\min}

\DeclareMathOperator{\rk}{rank}

\usepackage[colorinlistoftodos,prependcaption,textsize=tiny]{todonotes}

\begin{document}

\title{Collective Matrix Completion}

\author{%
Mokhtar Z. Alaya \\
Modal’X, UPL, Univ Paris Nanterre, \\
F92000 Nanterre France\\
\texttt{mokhtarzahdi.alaya@gmail.com} 
\and
Olga Klopp\\
ESSEC Business School \& CREST\\
F95021 Cergy France\\
\texttt{kloppolga@math.cnrs.fr}
}


\maketitle




\begin{abstract}%
Matrix completion aims to reconstruct a data matrix based on observations of a small number of its entries. 
Usually in matrix completion a single matrix is considered, which can be, for example, a rating matrix in recommendation system. 
However, in practical situations, data is often obtained from multiple sources which results in a collection of matrices rather than a single one. 
In this work, we consider the problem of collective matrix completion with multiple and heterogeneous matrices, which can be count, binary, continuous, etc.
We first investigate the setting where, for each source, the matrix entries are sampled from an exponential family distribution. 
Then, we relax the assumption of exponential family distribution for the noise.
In this setting, we do not assume any specific model for the observations. 
The estimation procedures are based on minimizing the sum of a goodness-of-fit term and the nuclear norm penalization of the whole collective matrix.
We prove that the proposed estimators achieve fast rates of convergence under the two considered settings and we corroborate our results with numerical experiments.
\end{abstract}

\noindent%

\section{Introduction}

Completing large-scale matrices has recently attracted great interest in machine learning and data mining since it appears in a wide spectrum of applications such as recommender systems~\citep{Koren09,bobadilla2013}, collaborative filtering (Netflix challenge)~\citep{goldberg92,rennie05}, sensor network localization~\citep{so05,drineas06,oh10}, system identification~\citep{liu09}, image processing~\citep{Hu2013-trruncated6389682}, among many others.
The basic principle of matrix completion  consists in recovering all the entries of an unknown data matrix from incomplete and noisy observations of its entries.

To address the high-dimensionality in matrix completion problem, statistical inference based on low-rank constraint is now an ubiquitous technique for recovering the underlying data matrix.
Thus, matrix completion can be formulated as minimizing the rank of the matrix given a random sample of its entries.
However, this rank minimization problem is in general NP-hard due to the combinatorial nature of the rank function~\citep{fazel2001,fazelPhD-2000}.
To alleviate this problem and make it tractable, convex relaxation strategies were proposed, e.g., the nuclear norm relaxation~\citep{srebo2005,candesPower2010,recht2010a,negahban2011,klopp2014} or the max-norm relaxation~\citep{cai2016}. 
Among those surrogate approximations, nuclear norm, which is defined as the sum of the singular values of the matrix or the $\ell_1$-norm of its spectrum, is probably the most widely used penalty for low-rank matrix estimation, since it is the tightest convex lower bound of the rank~\citep{fazel2001}.

\paragraph{Motivations.} 

Classical matrix completion focus on a single matrix, whereas in practical situations data is often obtained from a collection of matrices that may cover multiple and heterogeneous sources.
For example, in e-commerce users express their feedback for different items such as books, movies, music, etc. 
In social networks like Facebook and Twitter users often share their opinions and interests on a variety of topics (politics, social events, health).
In this examples, informations from multiple sources can be viewed as a collection of matrices coupled through a common set of users.

Rather than exploiting user preference data from each source independently, it may be beneficial to leverage all the available user data provided by various sources in order to generate more encompassing user models~\citep{cantador2015}. 
For instance, some recommender system runs into the so-called {\it cold-start problem}~\citep{lam2008-coldstart}. 
A user is new or ``cold'' in a source when he has few to none rated items.
Such user may have a rating history in auxiliary sources and we can use his profile in the auxiliary sources to recommend relevant items in the target source.
For example, a user's favorite movie genres may be derived from his  favorite book genres. 
Therefore, this shared structure among the sources can be useful to get better predictions~\citep{singh2008,bouchard13,gunasekar2016}.  

More generally speaking, collective matrix completion finds a natural application in the problem of recommender system with side information. In this problem, in addition to the conventional user-item matrix, it is assumed that we have side information about each user~\citep{chiang2015NIPS,jain2013,fithian2018,agarwal2011}. For example, in blog recommendation task, we may have access to user
generated content (images, tags and text) or user activity (e.g., likes and reblogs). Such side information may be used to improve the quality of recommendation of blogs of interest~\citep{tumblr}. 

Based on the type of available side information, various methods for recommender systems with side information have been proposed. It can be user generated content~\citep{armentano, hannon}, user/item
profile or attribute~\citep{agarwal2011}, social network~\citep{jamali,ma} and context information~\citep{natarajan}. A very interesting surveys of the
state-of-the-art methods can be found in~\citep{fithian2018,natarajan}.

On the other hand, our framework includes the model of Mixed Data Frames with missing observations~\citep{pages2014multiple,udell}. Here matrices collect categorical, numerical and count observations. They appear in numerous applications including in ecology, patient records in health care~\citep{gunasekar2016}, quantitative gene expression values~\citep{natarajan2014,zitnik2014MatrixFD,zitnik2015}, and also in recommender systems and 
survey data.

\paragraph{Main contributions and related literature.} 

In this paper, we extend the theory of low-rank matrix completion to a collection of multiple and heterogeneous matrices.
We first consider general matrix completion setting where we assume that for each matrix its 
entries are sampled from natural exponential distributions~\citep{lehmCase98}.
In this setting, we may have Gaussian distribution for continuous data; Bernoulli for binary data; Poisson for count-data, etc.
In a second part, we relax the assumption of exponential family distribution for the noise and we do not assume any specific model for the observations. 
This approach is more popular and widely used in machine learning. 
The proposed estimation procedure is based on minimizing the sum of a goodness-of-fit term and the nuclear norm penalization of the whole collective matrix.
The key challenge in our analysis is to use joint low-rank structure and 
our algorithm is far from the trivial one which consists in estimating each source matrix separately.
We provide theoretical guarantees on our estimation method and show that the collective approach provides faster rate of convergences.
We further corroborate our theoretical findings through simulated experiments.

Previous works on collective matrix completion are mainly based on matrix factorization~\citep{srebo2005}.
In a nutshell, this approach fits the target matrix as the product of two low-rank matrices.
Matrix factorization gives rise to non-convex optimization problems and  
its theoretical understanding is quite limited.  
For example, \cite{singh2008} proposed the collective matrix factorization that jointly factorizes multiple matrices sharing latent factors. As in our setting, each matrix can have a different value type and error distribution. In \cite{singh2008}, the authors use Bregman divergences to measure the error and extend standard alternating projection algorithms to this setting. They consider a quite general setting which includes as a particular case the nuclear norm penalization approach that we study in the present paper. They do not provide any theoretical guarantee.
A Bayesian model for collective matrix factorization was proposed in~\cite{singh2010}.
\cite{horii2014} and \cite{xu16aligned} also consider collective matrix factorization 
and investigate the strength of the relation among the source matrices. 
Their estimation procedure is based on penalization by the sum of the nuclear norms of the sources.
The convex formulation for collective matrix factorization was proposed in~\cite{bouchard13} where the authors consider a  general situation when the set of matrices do not necessarily have a common set of rows/columns. When this is the case, the estimator proposed in~\cite{bouchard13} is quite similar to ours. Their algorithm is based on the iterative Singular Value Thresholding and the authors conduct empirical evaluations of this approach on two real data sets. 

Most of the previous papers focus on the algorithmic side without providing theoretical guarantees for the collective approach. One exception is the paper by~\cite{gunasekar15consistent} where the authors prove consistency of the estimate under two observation models:
noise-free and additive noise models. Their estimation procedure is based on minimizing the least squares loss penalized by the nuclear norm. To prove the consistency of their estimator,~\cite{gunasekar15consistent} assume that all the source matrices share the same low-rank factor.
They consider the uniform sampling scheme for the observations (see Assumptions 1 and 4 in~\cite{gunasekar15consistent}). Uniform sampling is an usual assumption in matrix completion literature (see, e.g.,~\citep{candesPower2010,candes2009,davenport14}). This assumption is restrictive in many applications such as recommendations systems. The theoretical analysis in the present paper is carried out for general sampling distributions. 

Similar to our setting, matrix completion with side information explores  the available user data provided by various sources. For instance \cite{jain2013} and \cite{xu2013NIPS} introduce the so-called Inductive Matrix Completion (IMC). It  models side information as knowledge of feature spaces.
They show that if the features are perfect (e.,g., see Definition 1 in~\cite{chiang2018jmlr} for perfect side information), the sample complexity can be reduced. More precisely, in works on matrix completion with side information, it is usually assumed that one has partially observed low-rank matrix of interest $\bM\in \mathbb{R}^{d_1\times d_2}$ and, additionally, one has access to two matrices of features $\bA\in \mathbb{R}^{d_1\times r_1}$ and $\bB\in \mathbb{R}^{d_2\times r_2}$ where each row of $\bA$ (or $\bB$) denotes the feature of the $i$-th row (or column) entity of $\bM$, $r_i<d_i$ for $i=1,2$ and $\bM=\bA\bZ\bB^{T}$ . The main difference with our setting is that, here, $\bA$ and $\bB$ are assumed to be fully observed while our model allows also missing observations for the set of features.
The perfect side information assumption is strong and hard to meet in practice.   
~\cite{chiang2015NIPS} relaxed it by assuming that the side information may  be noisy (not perfect). In this approach, referred as DirtyIMC, they assume that the unknown matrix is modeled as $\bM=\bA\bZ\bB^{T}+ \bN$ where the residual matrix $\bN$ models imperfections and noise in the features.
	
Several works  consider matrix completion side information. For example,  \cite{chiang2015NIPS} proposes a method based on  penalization by the sum of the nuclear norms of $\bM$ and of each feature. Our method is based on the penalization by the nuclear norm of the whole matrix built of the matrix $\bM$ and the features $\bA$ and $\bB$.  In \cite{jain2013}, the authors study the problem of low-rank matrix estimation using rank one measurements. In the noise-free setting, they assume that all the features are known and that the matrices of features are incoherent. The method proposed in \cite{jain2013} is based on non-convex matrix factorization. In \cite{fithian2018}, the authors consider a general framework for reduced-rank modeling of matrix-valued data. They use a generalized weighted nuclear norm penalty where the matrix is multiplied by positive semidefinite matrices $P$ and $Q$ which depend on the matrix of features. In \cite{agarwal2011}, the authors introduce a per-item user covariate logistic regression model augmenting  with user-specific random effects. Their approach is based on	a multilevel hierarchical model.

In the case of the heterogeneous data coming from different sources, these approaches can be applied  for recovering each  source separately. In contrast, our approach aims at collecting all the available information in a single matrix which results in faster rates of convergence. On the other hand, popular algorithms for matrix completion with side information, such as  Maxide in~\cite{xu2013NIPS} and AltMin in~\cite{jain2013}, are based on the least square loss which could be not suitable for data  coming from non-Gaussian distributions.

If we consider a single matrix, our model includes as particular case $1$-bit matrix completion and, more generally, matrix completion with exponential family noise.
$1$-bit matrix completion was first studied in~\cite{davenport14}, where the observed entries are assumed to be sampled uniformly at random.
This problem was also studied among others by~\citep{cai2013jmlr-onebitmaxnorm,klopp2015EJS-adaptive-onebit,alquier2017}. 
Matrix completion with exponential family noise (for a single matrix) was previously considered in~\cite{pmlr-v40-lafond15} and~\cite{gunasekar2014jmlr}. 
In these papers authors assume sampling with replacement where there can be multiple observations for the same entry. 
In the present paper, we consider more natural setting for matrix completion where each entry may be observed at most once.
Our result improves the known results on $1$-bit matrix completion 
and on matrix completion with exponential family noise. 
In particular, we obtain exact minimax optimal rate of convergence for $1$-bit matrix completion and matrix completion with exponential noise which was known up to a logarithmic factor (for more details see Remark~\ref{example-1bitMC} in Section~\ref{section-exponential-noise}).

\paragraph{Organization of the paper.}

The remainder of the paper is organized as follow.
In Section~\ref{subsection-notations}, we introduce basic notation and definitions.
Section~\ref{sec:collective_exponential_family_matrix_completion} sets up the formalism for the collective matrix completion. 
In Section~\ref{section-exponential-noise}, we investigate the exponential family noise model.
In Section~\ref{section-distribution-free}, we study distribution-free setup and we provide the upper bound on the excess risk. 
To verify the theoretical findings, we corroborate our results with numerical experiments in Section~\ref{section-numeric-expes}, where we present an efficient iterative algorithm that solves the maximum likelihood approximately.
The proofs of the main results and key technical lemmas are postponed to the appendices.  

\subsection{Preliminaries}
\label{subsection-notations}

For the reader's convenience, we provide a brief summary of the standard notation and the definitions that will be frequently used throughout the paper. 

\paragraph{Notation.}
For any positive integer $m$, we use $[m]$ to denote  $\{1, \ldots, m\}.$ 
We use capital bold symbols such as $\bX, \bY, \bA,$ to denote matrices. 
For a matrix $\bA,$ we denote its $(i,j)$-th entry by $A_{ij}$.
As usual, let $\norm{\bA}_F = \sqrt{\sum_{i,j}A_{ij}^2}$ be the Frobenius norm and let $\norm{\bA}_\infty  = \max_{i,j}|A_{ij}|$ denote the elementwise $\ell_\infty$-norm. 
Additionally, $\norm{\bA}_{*}$ stands for the nuclear norm (trace norm), that is $\norm{\bA}_{*} = \sum_i \sigma_i(\bA)$ where $\sigma_1(\bA) \geq \sigma_2(\bA) \geq \cdots$ are singular values of $\bA$, and $\norm{\bA} = \sigma_1(\bA)$ to denote the operator norm. 
The inner product between two matrices is denoted by $\inr{\bA, \bB} = \text{tr}(\bA^\top\bB) = \sum_{ij}A_{ij}B_{ij}$, where $\text{tr}(\cdot)$ denotes the trace of a matrix. 
We write $\partial\Psi$ the subdifferential mapping of a convex functional $\Psi$.
Given two real numbers $a$ and $b$, we write $a\vee b = \max(a,b)$ and $a\wedge b = \min(a,b).$
The symbols $\P$ and $\E$ denote generic probability and expectation operators whose distribution is determined from the context.
The notation $\cst$ will be used to denote positive constant, that might change from one instance to the other.

\begin{definition}
A distribution of a random variable $X$ is said to belong to the natural exponential family, if its probability density function characterized by the parameter ${\eta}$ is given by:
\begin{equation*} %
X|\eta \sim f_{h,G}(x|\eta) = h(x)\exp\big({\eta x} - G({\eta})\big), 
\end{equation*}
where $h$ is a nonnegative function, called the base measure function, which is independent of the parameter $\eta$.
The function $G(\eta)$ is strictly convex, and is called the $\log$-partition function, or the cumulant function. 
This function uniquely defines a particular member distribution of the exponential family, and can be computed as: $G(\eta) = \log\big(\int_{} h(x) \exp({\eta x})dx\big)$.
\end{definition}

If $G$ is smooth enough, we have that $\E[X] = {G}'(\eta)$ and $\V ar[X] = {G}''(\eta),$
where $G'$ stands for the derivative of $G$.
The exponential family encompasses a wide large of standard distributions such as: 

\begin{itemize}
\item Normal, $\mathcal{N}(\mu, \sigma^2)$ (known $\sigma$), is typically used to model continuous data, with natural parameter $\eta = \frac{\mu}{\sigma^2}$ and $G(\eta) = \frac{\sigma^2}{2}\eta^2$.
\item Gamma, $\Gamma(\lambda, \alpha)$ (known $\alpha$), is often used to model positive valued continuous data, with natural parameter $\eta=-{\lambda}$ and $G(\eta) = -\alpha\log(-\eta)$.

\item Negative binomial, $\mathcal{NB}(p, r)$ (known $r$), is a popular distribution to model 
overdispersed count data, whose variance is larger than their mean,
with natural parameter $\eta = \log(1 - p)$ and  $G(\eta) = -r\log(1 -\exp(\eta))$.

\item Binomial, $\mathcal{B}(p,N)$ (known $N$), is used to model number of successes in $N$ trials, with natural parameter $\eta = \log(\frac{p}{1-p})$ (logit function) and $G(\eta) = N\log(1 + \exp(\eta))$.
\item Poisson, $\mathcal{P}(\lambda)$, is used to model count data, with natural parameter $\eta = \log(\lambda)$ and $G(\eta) =\exp(\eta)$.
\end{itemize}

Exponential, chi-squared, Rayleigh, Bernoulli and geometric distributions are special cases of the above five distributions.

\begin{definition}
Let $S$ be a closed convex subset of $\R^m$ and $\Phi: S \subset \textbf{dom}(\Phi)\rightarrow \R$ a continuously-differentiable and strictly convex function.
The Bregman divergence associated with $\Phi$~\citep{bregman1967,censor1997} $d_\Phi: S \times S \rightarrow [0, \infty)$ is defined as
\begin{equation*}
d_\Phi(x, y) = \Phi(x) - \Phi(y) - \inr{x - y, \nabla \Phi(y)},
\end{equation*}
where $\nabla \Phi(y)$ represents the gradient vector of $\Phi$ evaluated at $y$.
\end{definition}

The value of the Bregman divergence $d_\Phi(x, y)$ can be viewed as the difference between the value of $\Phi$ at $x$ and the first Taylor expansion of $\Phi$ around $y$ evaluated at point $x$.
For exponential family distributions, the Bregman divergence corresponds to the Kullback-Leibler divergence~\citep{banerjee2005} with $\Phi=G$.

\section{Collective matrix completion} 
\label{sec:collective_exponential_family_matrix_completion} 

Assume that we observe a collection of matrices ${\bcX} = (\bX^1, \ldots, \bX^V)$. 
In this collection components $\bX^v \in \R^{d_u \times d_v}$ have a common set of rows.
This common set of rows corresponds, for example, to a common set of users in a recommendation system.
The set of columns of each matrix $\bX^v$ corresponds to a different type of entity. 
In the case of recommender system it can be books, films, video game, etc.
Then, the entries of each matrix $\bX^v$ corresponds to the user's rankings for this particular type of products.

We assume that the distribution of each matrix $\bX^v$ depends on the matrix of parameters $\bM^v$.
This distribution can be different for different $v$.
For instance, we can have binary observations for one matrix $\bX^{v_1}$ with entries which correspond, for example, to like/dislike labels for a certain type of products, multinomial for another matrix $\bX^{v_2}$ with ranking going from $1$ to $5$ and Gaussian for a third matrix $\bX^{v_3}$. 

As it happens in many applications, we assume that for each matrix $\bX^v$ we observe only a small subset of its entries. 
We consider the following model: for $v\in[V]$ and $(i,j) \in [d_u]\times [d_v]$, let $B^v_{ij}$ be independent Bernoulli random variables with parameter $\pi^v_{ij}$. 
We suppose that $B^v_{ij}$ are independent from $X^v_{ij}$.
Then, we observe $Y^v_{ij} = B^v_{ij}X^v_{ij}$.
We can think of the $B^v_{ij}$ as masked variables.
If $B^v_{ij} =1$, we observe the corresponding entry of $\bX^v$, and when $B^v_{ij}=0$, we have a missing observation.

In the simplest situation each coefficient is observed with the same probability, i.e. for every $v\in[V]$ and $(i,j)\in [d_u] \times [d_v], \pi^v_{ij} = \pi$.
In many practical applications, this assumption is not realistic.  
For example, for a recommendation system, some users are more active than others and some items are more popular than others and thus rated more frequently.
Hence, the sampling distribution is in fact non-uniform.
In the present paper, we consider general sampling model where we only assume that each entry is observed with a positive probability:

\begin{assumption}
\label{assump-prob-items}
Assume that there exists a positive constant $0 < p < 1$ such that
\begin{equation*}
\min_{v\in[V]}\min_{(i,j)\in[d_u]\times[d_v]}\pi^v_{ij} \geq p.
\end{equation*}
\end{assumption}

Let $\Pi$ denotes the joint distribution of the Bernoulli variables $\big\{B_{ij}^v: (i,j)\in [d_u]\times [d_v], v\in[V]\big\}$. 
For any matrix $\bcA \in \R^{d_u\times D}$ where $D=\sum_{v\in[V]}d_v$, we define the weighted Frobenius norm 
\begin{align*}
\norm{\bcA}_{\Pi,F}^2 = \sum_{v\in[V]}\sum_{(i,j)\in[d_u]\times[d_v]}\pi^v_{ij}(A^v_{ij})^2.
\end{align*}
Assumption~\ref{assump-prob-items} implies $\norm{\bcA}_{\Pi,F}^2 \geq p \norm{\bcA}_{F}^2.$
For each $v\in[V]$ let us denote $\pi^v_{i\bcdot} = \sum_{j=1}^{d_v}\pi^v_{ij}$ and $\pi^v_{\bcdot j} =\sum_{i=1}^{d_u}\pi^v_{ij}$. 
Note we can easily get an estimations of $\pi^v_{i\bcdot}$ and $\pi^v_{\bcdot j}$ using the empirical frequencies:
\begin{equation*}
\widehat{\pi^v_{i\bcdot}} = \sum_{j\in[d_v]} B_{ij}^v
\quad \text {and} \quad \widehat{\pi^v_{\bcdot j}} = \sum_{i\in[d_u]} B_{ij}^v.
\end{equation*}
Let $\pi_{i\bcdot} = \sum_{v\in[V]}\pi^v_{i\bcdot}$, $\pi_{\bcdot j} =\max_{v\in[V]} \pi^v_{\bcdot j}$, and $\mu$ be an upper bound of its maximum, that is 
\begin{equation}
\label{upper-bound-marginal}
\max\limits_{(i,j) \in [d_u]\times [d_v]}(\pi_{i\bcdot}, \pi_{\bcdot j}) \leq \mu.
\end{equation}

\section{Exponential family noise}
\label{section-exponential-noise}

In this section we assume that for each $v$ distribution of $\bX^v$ belongs to the exponential family, that is
\begin{equation*}
\label{exponetial-model}
{X}^v_{ij}|M^v_{ij}\sim f_{h^v,G^v}({X}^v_{ij}|M^v_{ij}) = h^v({X}^v_{ij})\exp\big({X}^v_{ij}M^v_{ij} - G^v(M^v_{ij})\big).
\end{equation*}

We denote $\bcM = (\bM^1, \ldots, \bM^V)$  and let $\gamma$ be an upper bound on the sup-norm of $\bcM$, that is $\gamma = |\gamma_1| \vee |\gamma_2|$, where $\gamma_1 \leq M^v_{ij} \leq \gamma_2$  for every $v\in [V]$ and $(i,j) \in [d_u] \times [d_v]$.
Hereafter, we denote by $\mathscr{C}_\infty(\gamma) = \big\{\bcW \in \R^{d_u \times D}: \norm{\bcW}_\infty \leq \gamma\big\}$, the $\ell_\infty$-norm ball with radius $\gamma$ in the space $\R^{d_u \times D}$.
We need the following assumptions on densities $f_{h^v,G^v}$:

\begin{assumption}
\label{assump-Gv-bound}
For each $v\in[V]$, we assume that the function $G^v(\cdot)$ is twice differentiable and there exits two constants $L^2_\gamma, U_\gamma^2$ satisfying:
\begin{equation}
\label{assum-frist-claim}
\sup_{\eta \in[-\gamma - \frac{1}{K}, \gamma + \frac{1}{K}]}({G^v})''(\eta) \leq U^2_\gamma,
\end{equation}
and 
\begin{equation}
\label{assum-second-claim}
\inf\limits_{\eta  \in[-\gamma - \frac{1}{K}, \gamma + \frac{1}{K}]}(G^v)''(\eta) \geq L^2_\gamma,
\end{equation}
for some $K>0$.
\end{assumption}

The first statement, \eqref{assum-frist-claim}, in Assumption~\ref{assump-Gv-bound} ensures that the distributions of $X^v_{ij}$ have uniformly bounded variances and sub-exponential tails (see Lemma~\ref{lemma-subgaussian-tail-X} in Appendix~\ref{appendix-sub-exponentail-RV}). 
The second one, \eqref{assum-second-claim}, is the strong convexity condition satisfied by the log-partition function $G^v$.
This assumption is satisfied for most standard distributions presented in the previous section. 
In Table~\ref{table:glm}, we list the corresponding constants in Assumption~\ref{assump-Gv-bound}.

\begin{table}[htbp]
\centering
\begin{tabular}{ccccccc}
\toprule
Model &  $(G^v)'(\eta)$ & $(G^v)''(\eta)$ & $ L^2_\gamma$ & $U^2_\gamma$\\
\midrule
Normal & $\sigma^2 \eta$ & $\sigma^2$ & $\sigma^2$ & $\sigma^2$ \\
Binomial & $\frac{Ne^\eta}{1+e^\eta}$& $\frac{Ne^\eta}{(1+e^\eta)^2}$  & $\frac{Ne^{-(\gamma + \frac{1}{K})}}{(1 + e^{\gamma + \frac{1}{K}})^2}$ & $\frac{N}{4}$\\
Gamma (if $\gamma_1\gamma_2 >0$) & $-\frac{\alpha}{\eta}$ & $\frac{\alpha}{\eta^2}$ & $\frac{\alpha}{(\gamma + \frac{1}{K})^2}$ & $\frac{\alpha}{(|\gamma_1|\wedge |\gamma_2|)^2}$ \\
Negative binomial & $\frac{re^\eta}{1 -e^\eta}$& $\frac{re^\eta}{(1-e^\eta)^2}$  & $\frac{re^{-(\gamma + \frac{1}{K})}}{(1 - e^{-(\gamma + \frac{1}{K})})^2}$ & $\frac{re^{(\gamma + \frac{1}{K})}}{(1 - e^{\gamma + \frac{1}{K}})^2}$\\
Poisson & $e^\eta$ & $e^\eta$ & $e^{-(\gamma + \frac{1}{K})}$  & $e^{(\gamma + \frac{1}{K})}$\\
\bottomrule
\end{tabular}
\captionsetup{labelformat=default, skip=10pt}
\caption[table]{Examples of the corresponding constants $L_\gamma^2$ and $U^2_\gamma$ from Assumption~\ref{assump-Gv-bound}.}
\label{table:glm}
\end{table}

\subsection{Estimation procedure}
\label{sec:theoretical_guarantees}

To estimate the collection of matrices of parameters $\bcM = (\bM^1, \ldots, \bM^V)$, we use penalized negative log-likelihood.
Let $\bcW \in \R^{d_u\times D}$, we divide it in $V$ blocks $\bW^v \in \R^{d_u\times d_v}$:
$\bcW = (\bW^1, \ldots, \bW^V)$.
Given observations $\bcY = (\bY^1, \ldots, \bY^V)$, we write the negative log-likelihood as 
\begin{equation*}
\mathscr{L}_{\bcY}(\bcW)= -\frac{1}{d_uD} \sum_{v \in [V]}\sum_{(i,j)\in [d_u]\times [d_v]} B_{ij}^v\big(Y_{ij}^vW^v_{ij} - G^v(W^v_{ij})\big).
\end{equation*}
The nuclear norm penalized estimator $\widehat{\bcM}$ of $\bcM$ is defined as follows: 
\begin{equation}
\label{def-estimator}
\widehat{\bcM} = (\widehat{\bM}^1, \ldots, \widehat{\bM}^V) = \argmin_{\bcW \in \mathscr{C}_\infty(\gamma)}\mathscr{L}_{\bcY}(\bcW) + \lambda \norm{\bcW}_*,
\end{equation}
where $\lambda$ is a positive regularization parameter that balances the trade-off between model fit and privileging a low-rank solution. 
Namely, for large value of $\lambda$ the rank of the estimator $\widehat{\bcM}$ is expected to be small. 

Let the collection of matrices $(E^v_{11}, \ldots, E^v_{d_ud_v})$ form the canonical basis in the space of matrices of size $d_u \times d_v$.
The entry of $(E^v_{ij})$ is $0$ everywhere except for the $(i,j)$-th entry where it equals to $1.$
For $(\varepsilon^v_{ij})_{}$, an $i.i.d$ Rademacher sequence, we define $\bSigma_R = (\bSigma^1_R, \ldots, \bSigma^V_R)$ 
where for all $v\in [V]$
\begin{equation*}
\bSigma^v_R = \frac{1}{d_uD}\sum_{(i,j) \in [d_u]\times[d_v]}\varepsilon^v_{ij}B_{ij}^vE^v_{ij}.
\end{equation*}

We now state the main result concerning the recovery of $\bcM$.
Theorem~\ref{theorem1} gives a general upper bound on the estimation error of $\widehat{\bcM}$ defined by~\eqref{def-estimator}. 
Its proof is postponed in Appendix~\ref{proof-theorem1}. 

\begin{theorem}
\label{theorem1}
Assume that Assumptions~\ref{assump-prob-items} and~\ref{assump-Gv-bound} hold, and $\lambda \geq 2 \norm{\nabla\mathscr{L}_{\bcY}(\bcM)}.$
Then, with probability exceeding $1 - 4/(d_u+D)$ we have
\begin{align*}
  \frac{1}{d_uD}\norm{\widehat{\bcM} - \bcM}^2_{\Pi,F} \leq  \frac{\cst}{p}\max\Big\{d_uD\rk(\bcM)\Big(\frac{\lambda^2}{L_\gamma^4} + \gamma^2(\E[\norm{\bSigma_R}])^2\Big), \frac{\gamma^2\log(d_u +D)}{{d_uD}}\Big\},
\end{align*} 
where $\cst$ is a numerical constant.
\end{theorem}

Using Assumption~\ref{assump-prob-items}, Theorem~\ref{theorem1} implies the following bound on the estimation error measured in normalized Frobenius norm. 

\begin{corollary}
Under assumptions of Theorem~\ref{theorem1} and with probability exceeding $1 - 4/(d_u +D)$, we have 
\begin{equation*}
\frac{1}{d_uD}\norm{\widehat{\bcM} - \bcM}^2_{F} \leq  \frac{\cst}{p^2}\max\Big\{d_uD\rk(\bcM)\Big(\frac{\lambda^2}{L_\gamma^4} + \gamma^2(\E[\norm{\bSigma_R}])^2\Big), \frac{\gamma^2\log(d_u +D)}{{d_uD}}\Big\}.
\end{equation*}
\end{corollary}

In order to get a bound in a closed form we need to obtain a suitable upper bounds on $\E[\norm{\bSigma_R}]$ and on $\norm{\nabla\mathscr{L}_{\bcY}(\bcM)}$ with high probability.
Therefore we use the following two lemmas.

\begin{lemma}
\label{lemma-ctrl-sigmaR}
There exists an absolute constant $\cst$ such that 
\begin{equation*}
\label{ctrl-sigmaR}
\E[\norm{\bSigma_R}] \leq \cst\Big(\frac{\sqrt{\mu} + \sqrt{\log(d_u \wedge D)}}{d_uD}\Big).
\end{equation*}
\end{lemma}

\begin{lemma}
\label{lemma-lemma-ctrl-bSigma}
Let Assumption~\ref{assump-Gv-bound} holds. 
Then, there exists an absolute constant $\cst$ such that, with probability at least $1 - 4 /(d_u + D)$, we have
\begin{equation*}
\label{lemma-ctrl-bSigma}
\norm{\nabla\mathscr{L}_{\bcY}(\bcM)} \leq \cst \bigg(\frac{(U_\gamma \vee K)\big(\sqrt{\mu} + (\log(d_u \vee D))^{3/2}\big)}{d_uD} \bigg).
\end{equation*}
\end{lemma}

The proofs of Lemmas~\ref{lemma-ctrl-sigmaR} and~\ref{lemma-lemma-ctrl-bSigma} are postponed to Appendices~\ref{proof-lemma-ctrl-sigmaR} and~\ref{proof-lemma-ctrl-bSigma}.
Recall that the condition on $\lambda$ in Theorem~\ref{theorem1} is that  $\lambda \geq 2 \norm{\nabla\mathscr{L}_{\bcY}(\bcM)}.$
Using Lemma~\ref{lemma-lemma-ctrl-bSigma}, we can choose
\begin{equation*}
\lambda = 2\cst \frac{(U_\gamma \vee K) \big(\sqrt{\mu} + (\log(d_u \vee  D))^{3/2}\big)}{d_uD}.
\end{equation*}
With this choice of $\lambda$, we obtain the following theorem:

\begin{theorem}
\label{theorem2}
Let Assumptions~\ref{assump-prob-items} and~\ref{assump-Gv-bound} be satisfied.
Then, with probability exceeding $1 - 4/(d_u+D)$ we have
\begin{align*}
\frac{1}{d_uD}\norm{\widehat{\bcM} - \bcM}^2_{\Pi,F}
   \leq  \frac{\cst\rk(\bcM)}{pd_uD}&\Big(\gamma^2+ \frac{(U_\gamma \vee K)^2}{L^4_\gamma}\Big) \big(\mu+ \log^3(d_u \vee D)\big),
\end{align*}
and
\begin{align*}
\frac{1}{d_uD}\norm{\widehat{\bcM} - \bcM}^2_{F}\leq  \frac{\cst\rk(\bcM)}{p^2d_uD}
&\Big(\gamma^2+ \frac{(U_\gamma \vee K)^2}{L^4_\gamma}\Big) \big(\mu+ \log^3(d_u \vee D)\big),
\end{align*}
where $\cst$ is an absolute constant.
\end{theorem}

\begin{remark}
\label{remark-1-expnoise}
Note that the rate of convergence in Theorem~\ref{theorem2} has the following dominant term: 
\begin{equation*}
\label{dominant-order}
\frac{1}{d_uD}\norm{\widehat{\bcM} - \bcM}^2_{F} \lesssim \frac{\rk(\bcM)\mu}{p^2d_uD},
\end{equation*}
where the symbol $\lesssim$ means that the inequality holds up to a multiplicative constant.
If we assume that the sampling distribution is close to the uniform one, that is that there exists positive constants $c_1$ and $c_2$ such that for every $v\in[V]$ and $(i,j)\in[d_u]\times [d_v]$ we have $c_1p \leq \pi^v_{ij} \leq c_2p$, then Theorem~\ref{theorem2} yields
\begin{equation*}
\label{dominant-order}
\frac{1}{d_uD}\norm{\widehat{\bcM} - \bcM}^2_{F} \lesssim \frac{\rk(\bcM)}{p(d_u\wedge D)}.
\end{equation*}

If we complete each matrix separately, the error will be of the order $\sum_{v=1}^V \rk(\bM^v) /p(d_u\wedge D)$.
As $\rk(\bcM) \leq \sum_{v=1}^V \rk(\bM^v)$, the rate of convergence achieved by our estimator is faster compared to the penalization by the sum-nuclear-norm.

In order to get a small estimation error, $p$ should be larger than $\rk(\bcM)/(d_u\wedge D)$.
We denote $n=\sum_{v\in[V]}\sum_{(i,j)\in[d_u]\times[d_v]}\pi^v_{ij},$ the expected number of observations. 
Then, we get the following condition on $n$:
\begin{equation*}
n \geq \cst \rk(\bcM) (d_u\vee D).
\end{equation*}
\end{remark}

\begin{remark}
\label{example-1bitMC}
In $1$-bit matrix completion~\citep{davenport14,klopp2015EJS-adaptive-onebit,alquier2017}, instead of observing the actual entries of the unknown matrix $\bcM \in \R^{d \times D}$, for a random subset of its entries $\Omega$ we observe $\{Y_{ij}\in \{+1, -1\}:(i,j)\in\Omega\}$, where $Y_{ij}=1$ with probability $f(M_{ij})$ for some link-function $f$.
In~\cite{davenport14} the parameter $\bcM$ is estimated by minimizing the negative log-likelihood under the constraints $\norm{\bcM}_\infty \leq \gamma$ and $\norm{\bcM}_* \leq \gamma \sqrt{rdD}$ for some $r >0$.
Under the assumption that $\rk(\bcM) \leq r,$ the authors prove that 
\begin{equation}
\label{onebit-davenport}
\frac{1}{dD}\norm{\widehat{\bcM} - \bcM}^2_{F} \leq \cst_\gamma\sqrt{\frac{r(d \vee D)}{n}},
\end{equation}
where $\cst_\gamma$ is a constant depending on $\gamma$ (see Theorem 1 in \cite{davenport14}).
A similar result using max-norm minimization was obtained in~\cite{cai2013jmlr-onebitmaxnorm}.
In~\citep{klopp2015EJS-adaptive-onebit} the authors prove a faster rate.
Their upper bound (see Corollary 2 in~\cite{klopp2015EJS-adaptive-onebit}) is given by 
 \begin{equation}
 \label{onebit-klopplafondsalmon}
\frac{1}{dD}\norm{\widehat{\bcM} - \bcM}^2_{F} \leq \cst_\gamma \frac{\rk(\bcM)(d\vee D)\log(d \vee D)}{n}.
\end{equation}
In the particular case of $1$-bit matrix completion for a single matrix under uniform sampling scheme, Theorem~\ref{theorem2} implies the following bound:
\begin{equation*}
\frac{1}{dD}\norm{\widehat{\bcM} - \bcM}^2_{F} \leq \cst_\gamma \frac{\rk(\bcM)(d \vee D)}{n},
\end{equation*}
which improves~\eqref{onebit-klopplafondsalmon} by a logarithmic factor.
Furthermore,~\cite{klopp2015EJS-adaptive-onebit} provide $\rk(M)(d\vee D)/n$ as the lower bound for $1$-bit matrix completion (see Theorem 3 in~\cite{klopp2015EJS-adaptive-onebit}).
So our result answers the important theoretical question what is the exact minimax rate of convergence for $1$-bit matrix completion which was previously known up to a logarithmic factor.

In a more general setting of matrix completion with exponential family noise, the minimax optimal rate of convergence was also known only up to logarithmic factor (see~\cite{pmlr-v40-lafond15}). Our result provides the exact minimax optimal rate in this more general setting too.
It is easy to see, by inspection of the proof of the lower bound in~\cite{pmlr-v40-lafond15}, that the upper bound provided by Theorem~\ref{theorem2} is optimal for the collective matrix completion.

\end{remark}

\begin{remark}
Note that our estimation method is based on the minimization of the nuclear-norm of the whole collective matrix $\bcM$. 
Another possibility is to penalize by the sum of the nuclear norms $\sum_{v\in[V]} \norm{\bM^v}_*$ (see, e.g., \cite{klopp2015EJS-adaptive-onebit}).
This approach  consists in estimating each component matrix independently. 
\end{remark}

\section{General losses}
\label{section-distribution-free}
In the previous section we assume that the link functions $G^v$ are known.
This assumption is not realistic in many applications.
In this section we relax this assumption in the sense that we do not assume any specific model for the observations.
Recall that our observations are a collection of partially observed matrices $\bY^v = (B_{ij}^vX_{i,j}^v)\in \R^{d_u \times d_v}$ for $v=1, \ldots, V$ and $\bX^v = (X_{ij}^v) \in \R^{d_u\times d_v}$.
We are interested in the problem of prediction of the entries of the collective matrix $\bcX =(\bX^1, \ldots, \bX^V)$.
We consider the risk of estimating $\bX^v$ with a loss function $\ell^v$, which measures the discrepancy between the predicted and actual value with respect to the given observations.
We focus on non-negative convex loss functions that are Lipschitz:

\begin{assumption}
\label{assumption-lipshitzloss}
(Lipschitz loss function) 
For every $v\in [V]$, we assume that the loss function $\ell^v(y, \cdot)$ is $\rho_v$-Lipschitz in its second argument:
$|\ell^v(y,x) - \ell^v(y,x')| \leq \rho_v|x- x'|.$
\end{assumption} 

Some examples of the loss functions that are $1$-Lipschitz are: hinge loss $\ell(y, y') =  \max(0, 1 - yy')$, logistic loss $\ell(y, y') = \log(1 + \exp(-yy'))$, and quantile regression loss $\ell(y, y') = \ell_\tau(y'-y)$
where $\tau\in (0,1)$ and $\ell_\tau(z) = z (\tau - \ind{}(z\leq 0))$.

For a matrix $\bcM = (\bM^1, \ldots, \bM^V) \in \R^{d_u\times D}$, we define the empirical risk as 
\begin{equation*}
{R}_{\bcY}(\bcM)= \frac{1}{d_uD} \sum_{v \in [V]}\sum_{(i,j)\in [d_u]\times [d_v]} B^v_{ij}\ell^v(Y^v_{ij}, M^v_{ij}).
\end{equation*}
We define the oracle as:
\begin{equation}
\label{oracle-ranking}
\ddt{\bcM} = \big(\dt{\bM}^1, \ldots, \dt{\bM}^V\big) = \argmin_{\bcQ \in \mathscr{C}_\infty(\gamma)} {R}_{}(\bcQ)
\end{equation}
where ${R}_{}(\bcQ) = \E[{R}_{\bcY}(\bcQ)]$.
Here the expectation is taken over the joint distribution of $\{(Y_{ij}^v, B_{ij}^v): (i,j)\in[d_u]\times[d_v] \text{ and } v\in[V]\}.$
We use machine learning approach and will provide an estimator $\widehat{\bcM}$ that predicts almost as well as $\ddt{\bcM}$.
Thus we will consider excess risk ${R}_{}(\widehat{\bcM}) - {R}_{}(\ddt{\bcM})$.
By construction, the excess risk is always positive.

For a tuning parameter $\Lambda > 0$, the nuclear norm penalized estimator $\widehat{\bcM}$ is defined as
\begin{equation}
\label{def-estimator-Lip-loss}
\widehat{\bcM} \in \argmin_{\bcQ \in \mathscr{C}_\infty(\gamma)} \big\{{R}_{\bcY}(\bcQ)+ \Lambda \norm{\bcQ}_*\big\}.
\end{equation}
We next turn to the assumption needed to establish an upper bound on the performance of the estimator $\widehat{\bcM}$ defined in~\eqref{def-estimator-Lip-loss}.
\begin{assumption}
\label{assumption-Bernstein-condition}
Assume that there exists a constant $\varsig > 0$ such that for every $\bcQ \in \mathscr{C}_\infty(\gamma)$, we have 
\begin{equation*}
 {R}_{}(\bcQ) - {R}_{}(\ddt{\bcM})\geq \frac{\varsig}{d_uD}\norm{\bcQ - \ddt{\bcM}}^{2}_{\Pi,F}.
\end{equation*}
\end{assumption}
This assumption has been extensively studied in the learning theory literature~\citep{MENDELSON2008380,zhang2004behaviourconsitensy,bartellet2004nips,alquier2017,elsener2017}, and it is called  ``Bernstein'' condition.
It is satisfied in various cases of loss function~\citep{alquier2017} and it ensures a sufficient convexity of the risk around the oracle defined in~\eqref{oracle-ranking}.
Note that when the loss function $\ell^v$ is strongly convex, the risk function inherits this property and automatically satisfies the margin condition.
In other cases, this condition requires strong assumptions on the distribution of the observations, for instance for hinge loss or quantile loss (see Section 6 in \cite{alquier2017}). 
The following result gives an upper bound on the excess risk of the estimator $\widehat{\bcM}$.
\begin{theorem}
\label{theorem-oracle-ranking}
Let Assumptions~\ref{assump-prob-items},~\ref{assumption-lipshitzloss} and~\ref{assumption-Bernstein-condition} hold and set $\rho=\max_{v\in[V]}\rho_v$. 
Suppose that $\Lambda \geq 2 \sup\{\norm{\bcG}: \bcG \in \partial{R}_{\bcY}(\ddt{\bcM})\}.$ Then, with probability at least $1 - 4/(d_u+D)$, we have 
\begin{align*}
R(\widehat{\bcM}) - R(\ddt{\bcM}) \leq \frac{\cst}{p} \max\Big\{\rk(\ddt{\bcM})d_uD\Big(\rho^{3/2}\sqrt{\gamma/\varsig}(\E[&\norm{\bSigma_R}])^2 + \frac{\Lambda^2}{\varsig} \Big),\\
& \frac{\big(\rho\gamma + \rho^{3/2}\sqrt{\gamma/{\varsig}}\big)\log(d_u+D)}{d_uD}\Big\}.\\
\end{align*}
\end{theorem}
Theorem~\ref{theorem-oracle-ranking}  gives a general upper bound on the prediction error of the estimator $\widehat{\bcM}$.
Its proof is presented in Appendix~\ref{proof-of-theorem-oracle-ranking}.
In order to get a bound in a closed form we need to obtain a suitable upper bounds on $\sup\{\norm{\bcG}: \bcG \in \partial({R}_{\bcY}(\ddt{\bcM}))\}$ with high probability.
\begin{lemma}
\label{lemma-upper-bound-bSigmastar}
Let Assumption~\ref{assumption-lipshitzloss} holds.
Then, there exists an absolute constant $\cst$ such that, with probability at least $1 - 4 /(d_u + D)$, we have
\begin{align*}
\norm{\bcG}\leq \cst\frac{\rho\big(\sqrt{\mu} + \sqrt{\log(d_u \vee D)}\big)}{d_uD},
\end{align*}
for all $\bcG \in \partial{R}_{\bcY}(\ddt{\bcM}).$
\end{lemma}
The proof of Lemma~\ref{lemma-upper-bound-bSigmastar} is given in Appendix~\ref{proof-lemma-upper-bound-bSigmastar}. 
Using Lemma~\ref{lemma-upper-bound-bSigmastar} , we can choose
\begin{equation*}
\Lambda = 2\cst\frac{\rho\big(\sqrt{\mu} + \sqrt{\log(d_u \vee D)}\big)}{d_uD}
\end{equation*}
and with this choice of $\Lambda$ and Lemma~\ref{lemma-ctrl-sigmaR}, we obtain the following theorem:
\begin{theorem}
\label{theorem-excess-risk}
Let Assumptions~\ref{assump-prob-items},~\ref{assumption-lipshitzloss} and~\ref{assumption-Bernstein-condition} hold. 
Then, we have 
\begin{align*}
R(\widehat{\bcM}) - R(\ddt{\bcM}) \leq \frac{\cst}{p} \rk(\ddt{\bcM})\frac{(\rho^2 +\rho^{3/2}\sqrt{\gamma/\varsig})(\mu + \log(d_u \vee D))}{d_uD},
\end{align*}
with probability at least $1 - 4/(d_u+D)$.
\end{theorem}
Using Assumption~\ref{assumption-Bernstein-condition}, we get the following corollary:
\begin{corollary}
With probability at least $1 - 4/(d_u+D)$, we have 
\begin{equation*} 
\frac{1}{d_uD}\norm{\widehat{\bcM} - \ddt{\bcM}}_{F}^2 \leq \frac{\cst }{p^2\varsig} \rk(\ddt{\bcM}) \frac{(\rho^2 +\rho^{3/2}\sqrt{\gamma/\varsig})(\mu + \log(d_u \vee D))}{d_uD}.
\end{equation*}
\end{corollary}

\paragraph{$1$-bit matrix completion.} 

In $1$-bit matrix completion with logistic (resp. hinge) loss, the Bernstein assumption is satisfied with $\varsig = 1/(4e^{2\gamma})$ (resp. $\varsig =2\tau$, for some $\tau$ that verifies $|\dt{M}^v_{ij} - 1/2| \geq \tau, \forall v\in[V], (i,j)\in[d_u]\times[d_v]$).
More details for these constants can be found in Propositions 6.1 and 6.3 in~\cite{alquier2017}.  
Then, the excess risk with respect to these two losses under the uniform sampling is given by: 
\begin{corollary}
\label{corollary:logisticloss-1bitMC-excessrisk}
With probability at least $1 - 4/(d_u+D)$, we have 
\begin{align*}
R(\widehat{\bcM}) - R(\ddt{\bcM}) \leq {\cst}\frac{\rk(\ddt{\bcM})}{p(d_u\wedge D)}.
\end{align*}
\end{corollary}
These results are obtained without a logarithmic factor, and it improves the ones given in Theorems 4.2 and 4.4 in~\cite{alquier2017}. 
The natural loss in this context is the $0/1$ loss which is often replaced by the hinge or the logistic loss.
We assume without loss of generality that $\gamma =1$, since the Bayes classifier has its entries in $[-1, 1]$, and 
we define the classification excess risk by:
\begin{equation*}
	R_{0/1}(\bcM) = \frac{1}{d_uD} \sum_{v\in[V]}\sum_{(i,j)\in[d_u]\times[d_v]}\pi_{ij}^v \P[X_{ij}^v \neq \text{sign}(M^v_{ij})],
\end{equation*}
for all $\bcM \in \R^{d_u \times D}.$ Using Theorem 2.1 in~\cite{zhang2004behaviourconsitensy}, we have 
\begin{equation*}
  R_{0/1}(\widehat{\bcM}) - R_{0/1}(\ddt{\bcM}) \leq \cst \sqrt{\frac{\rk(\ddt{\bcM})}{p(d_u\wedge D)}}.
\end{equation*}

\section{Numerical experiments}
\label{section-numeric-expes}

In this section, we first provide algorithmic details of the numerical procedure for solving the problem~\eqref{def-estimator}, then we conduct experiments on synthetic data to further illustrate the theoretical results of the collective matrix completion. 

\subsection{Algorithm}

The collective matrix completion problem~\eqref{def-estimator} is a semidefinite program (SDP), since it is a nuclear norm minimization problem with a convex feasible domain~\citep{fazel2001,srebo2005}. 
We may solve it, for example, via the interior-point method~\citep{Liu-Interirodoi:10.1137/090755436}. 
However, SDP solvers can handle a moderate dimensions, thus such formulation is not scalable due to the storage and computation complexity in low-rank matrix completion tasks.
In the following, we present an algorithm that solves the problem~\eqref{def-estimator} approximately and in a more efficient way than solving it as SDP.

\paragraph{Proximal Gradient.}

Problem~\eqref{def-estimator} can be solved by first-order optimization methods such as proximal gradient (PG) which has been popularly used for optimizations problems of the form of~\eqref{def-estimator}~\citep{Beck:2009:FIS:1658360.1658364,Nesterov2013,Parikh:2014:PA:2693612.2693613,jiye2009,mazumder2010SoftImpute,Yao:2015:AIS:2832747.2832807}.
When $\mathscr{L}_{\bcY}$ has $L$-Lipschitz continuous gradient, that is $\norm{\nabla\mathscr{L}_{\bcY}(\bcW) - \nabla\mathscr{L}_{\bcY}(\bcQ)}_F \leq L \norm{\bcW - \bcQ}_F$,
the PG generates a sequence of estimates $\{\bcW_t\}$ as  

\begin{align}
\label{pgd-algorithm}
\bcW_{t+1} &= \argmin_{\bcW}\mathscr{L}_{\bcY}(\bcW) + (\bcW - \bcW_t)^\top \nabla\mathscr{L}_{\bcY}(\bcW_t) + \frac{L}{2} \norm{\bcW - \bcW_t}_F^2 + \lambda\norm{\bcW}_*\nonumber\\
&= \text{prox}_{\frac{\lambda}{L}\norm{\cdot}_*}(\bcZ_t), \text{ where } \bcZ_t = \bcW_t - \frac 1L \nabla\mathscr{L}_{\bcY}(\bcW_t) 
\end{align}
and for any convex function $\Psi:\R^{d_u\times D} \mapsto \R$, the associated proximal operator at $\bcW \in \R^{d_u \times D}$ is defined as
\begin{equation*}
\text{prox}_\Psi(\bcW) = \argmin\big\{\frac 12 \norm{\bcW - \bcQ}_F^2 + \Psi(\bcQ): \bcQ \in \R^{d_u \times D}\big\}.
\end{equation*}
The proximal operator of the nuclear norm at $\bcW \in \R^{d_u \times D}$ corresponds to 
the singular value thresholding ({SVT}) operator of $\bcW$~\citep{can2010SVTAlgo}.
That is, assuming a singular value decomposition $\bcW = \bcU \bSigma \bcV^\top,$ where $\bcU \in \R^{d_u \times r}$, $\bcV \in \R^{D\times r}$ have orthonormal columns, $\bSigma = (\sigma_1, \ldots, \sigma_r)$, with $\sigma_1 \geq \cdots \geq \sigma_r > 0$ and 
$r = \rk(\bcW)$, we have 
\begin{equation}
\label{svt-prox}
\text{{SVT}}_{\lambda/L}(\bcW) = \bcU \text{diag}((\sigma_1 - \lambda/L)_{+}, \ldots, (\sigma_r -  \lambda/L)_{+})\bcV^\top,
\end{equation}
where $(a)_{+} = \max(a, 0)$. 

Although PG can be implemented easily, it converges slowly when the Lipschitz constant $L$ is large.
In such scenarios, the rate is $\bigO(1/T)$, where $T$ is the number of iterations~\citep{Parikh:2014:PA:2693612.2693613}.
Nevertheless, it can be accelerated by replacing $\bcZ_t$ in~\eqref{pgd-algorithm} with 
\begin{equation}
\label{acc-pgd}
\bcQ_t = (1 + \theta_t) \bcW_t - \theta_t \bcW_{t-1}, \quad \bcZ_t = \bcQ_t - \eta \nabla\mathscr{L}_{\bcY}(\bcQ_t).
\end{equation}
Several choices for $\theta_t$ can be used. The resultant accelerated proximal gradient (APG) (see Algorithm~\ref{algorithm-apgd}) converges with the optimal $\bigO(1/T^2)$ rate~\citep{Nesterov2013,Ji:2009:AGM:1553374.1553434}.

\LinesNotNumbered
\begin{algorithm}[htbp]
\SetNlSty{textbf}{}{.}
\DontPrintSemicolon
\caption{\small APG for Collective Matrix Completion}
\label{algorithm-apgd}
\nl \textbf{initialize:} $\bcW_0 = \bcW_1 = \bcY,$ and $ \alpha_0 = \alpha_1 = 1$.\\
\nl \For{$t=1, \ldots, T$} 
{
\nl     $\bcQ_t = \bcW_t + \frac{\alpha_{t-1} - 1}{\alpha_t}(\bcW_t - \bcW_{t-1});$\\
\nl     $\bcW_{t+1} = \text{{SVT}}_{\frac \lambda L}(\bcQ_t - \frac 1L\nabla\mathscr{L}_{\bcY}(\bcQ_t));$\\
\nl     $ \alpha_{t+1} = \frac{1}{2} (\sqrt{4\alpha_t^2 +1} + 1);$
}
\nl \Return{$\bcW_{T+1}$.}
\end{algorithm}

\paragraph{Approximate SVT~\citep{Yao:2015:AIS:2832747.2832807}.}
To compute $\bcW_{t+1}$ in the proximal step (SVT) in Algorithm~\ref{algorithm-apgd}, we need first perform {SVD} of $\bcZ_t$ given in~\eqref{acc-pgd}.
In general, obtaining the {SVD} of $d_u\times D$ matrix $\bcZ_t$ requires $\bigO((d_u \wedge D) d_uD)$ operations, because its most expensive steps are computing matrix-vector multiplications. 
Since the computation of the proximal operator of the nuclear norm given in~\eqref{svt-prox} does not require to do the full SVD, only a few singular values of $\bcZ_t$ which are larger than $\lambda/L$ are needed.
Assume that  there are $\hat k$ such singular values. As $\bcW_t$ converges to a low-rank solution $\bcW_*$, $\hat k$ will be small during iterating.
{The power method~\citep{Halkodoi:10.1137/090771806} at Algorithm~\ref{power-method} is a simple and efficient to capture subspace spanned by top-$k$ singular vectors for $\hat k \geq k$.}
Additionally, the power method also allows warm-start, which is particularly useful because the iterative nature of APG algorithm.
Once an approximation $\bcQ$ is found, we have $\textrm{SVT}_{\lambda/L}(\bcZ_t) = \bcQ \textrm{SVT}_{\lambda/L}(\bcQ^\top\bcZ_t)$ (see Proposition 3.1 in~\cite{Yao:2015:AIS:2832747.2832807}).
We therefore reduce the time complexity on SVT from $\bigO((d_u\wedge D)d_uD)$ to $\bigO(\hat k d_uD)$ which is much cheaper.

\LinesNotNumbered
\begin{algorithm}[htbp]
\SetNlSty{textbf}{}{.}
\DontPrintSemicolon
\caption{Power Method: \texttt{PowerMethod}$(\bcZ, \bcR, \epsilon)$}
\label{power-method}
\nl \textbf{input:} $\bcZ \in \R^{d_u\times D}$, initial $\bcR \in \R^{D\times k}$ for warm-start, tolerance $\delta$;\\
\nl \textbf{initialize} $\bcW_1 = \bcZ \bcR$;\\
\nl \For {$t=1, 2, \ldots, $} 
{
\nl   $\bcQ_{t+1} = \text{QR}(\bcW_t);$// QR denotes the QR factorization\\
\nl   $\bcW_{t+1} = \bcZ(\bcZ^\top \bcQ_{t+1})$;\\
\nl   \If {$\norm{\bcQ_{t+1}\bcQ_{t+1}^\top - \bcQ_{t}\bcQ_{t}^\top}_F \leq \delta$}
    {
      break;
    }   
}
\nl \Return{$\bcQ_{t+1}$.}
\end{algorithm}

Algorithm~\ref{prox-SVT} shows how to approximate $\textrm{SVT}_{\lambda/L}(\bcZ_t)$.
Let the target (exact) rank-$k$ SVD of $\bcZ_t$ be $\bcU_k\bSigma_k\bcV_k^\top$.
Step 1 first approximates $\bcU_k$ by the power method. 
In steps 2 to 5, a less expensive $\textrm{SVT}_{\lambda/L}(\bcQ^\top\bcZ_t)$ is obtained from~\eqref{svt-prox}. Finally, $\textrm{SVT}_{\lambda/L}(\bcZ_t)$ is recovered.

\LinesNotNumbered
\begin{algorithm}[hbtp]
\SetNlSty{textbf}{}{.}
\DontPrintSemicolon
\caption{Approximate {SVT}: \texttt{Approx-SVT}$(\bcZ, \bcR, \lambda, \delta)$}
\label{prox-SVT}
\nl \textbf{input:} $\bcZ \in \R^{d_u\times D}, \bcR \in \R^{D\times k},$ thresholds $\lambda$ and $\delta$;\\
\nl $\bcQ = \texttt{PowerMethod}(\bcZ, \bcR, \delta)$;\\
\nl $[\bcU, \bSigma, \bcV] = \text{{SVD}}(\bcQ^\top \bcZ)$;\\
\nl $\bcU = \{u_i | \sigma_i > \lambda\}$;\\
\nl $\bcV = \{v_i | \sigma_i > \lambda\}$;\\
\nl $\bSigma = \max(\bSigma - \lambda\bcI, \mathbf{0});$ // ($\bcI$ denotes the identity matrix)\\
\nl \Return{$\bcQ \bcU, \bSigma, \bcV$.}
\end{algorithm}

Hereafter, we denote the objective function in~\eqref{def-estimator} by $\mathscr{F}_{\lambda}(\bcW)$, that is $\mathscr{F}_{\lambda}(\bcW)=  \mathscr{L}_{\bcY}(\bcW) + \lambda \norm{\bcW}_*$, for any $\bcW \in\mathscr{C}(\gamma)$.
Recall that the gradient of the likelihood $\mathscr{L}_{\bcY}$ is written as 
\begin{equation*}
\nabla\mathscr{L}_{\bcY}(\bcW)= -\frac{1}{d_uD} \sum_{v \in [V]}\sum_{(i,j)\in [d_u]\times [d_v]} B_{ij}^v(Y_{ij}^v - (G^v)'(W^v_{ij})) E^v_{ij}.
\end{equation*}
By Assumption~\ref{assump-Gv-bound}, we have for any $\bcW, \bcQ \in \R^{d_u \times D}$
\begin{align*}
\norm{\nabla\mathscr{L}_{\bcY}(\bcW) - \nabla\mathscr{L}_{\bcY}(\bcQ)}_F^2 &= \frac{1}{(d_uD)^2}\sum_{v \in [V]}\sum_{(i,j)\in [d_u]\times [d_v]} \{B_{ij}^v((G^v)'(W^v_{ij}) - (G^v)'(Q^v_{ij}))\}^2\\
& \leq \frac{U_\gamma^2}{(d_uD)^2} \norm{\bcW - \bcQ}_F^2.
\end{align*}
This yields that $\mathscr{L}_{\bcY}$ has $L$-Lipschitz continuous gradient with $L = U_\gamma / (d_uD) \leq 1$.
In the following algorithm and the experimental setup, we choose to work with $L=1.$

\paragraph{Penalized Likelihood Accelerated Inexact Soft Impute (PLAIS-Impute).} 

We present here the main algorithm in this paper, referred to as {PLAIS-Impute}, which is tailored to solving our collective matrix completion problem.
The {PLAIS-Impute} is an adaption of the {AIS-Impute} 
algorithm in~\cite{Yao:2015:AIS:2832747.2832807} to the penalized likelihood completion problems.
Note that {AIS-Impute} is an accelerated proximal gradient algorithm  with further speed up based on approximate {SVD}.
However, it is dedicated only to square-loss goodness-of-fitting. 
The {PLAIS-Impute} is summarized in Algorithm~\ref{algorithm-aisoftimpute}.
The core steps are 10-12, where an approximate SVT is performed. 
Steps 10 and 11 use the column space of the last iterations ($\bcV_{t}$ and $\bcV_{t-1}$) to warm-start the power method.
For further speed up, a continuation strategy is employed in which $\lambda_t$ is initialized to a large value and then decreases gradually.
The algorithm is restarted (at the step 14) if the objective function $\mathscr{F}_{\lambda}$ starts to increase. 
As {AIS-Impute}, {PLAIS-Impute} shares both low-iteration complexity and fast $\bigO(1/T^2)$ convergence rate (see Theorem 3.4 in~\cite{Yao:2015:AIS:2832747.2832807}).

\LinesNotNumbered
\begin{algorithm}[htbp]
\SetNlSty{textbf}{}{.}
\DontPrintSemicolon %
\caption{PLAIS-Impute for Collective Matrix Completion} 
\label{algorithm-aisoftimpute}
\nl \textbf{input: }{observed collective matrix $\bcY$, parameter $\lambda$, decay parameter $\nu \in (0,1)$, tolerance $\varepsilon$;}\\
\nl $[\bcU_0, \lambda_0, \bcV_0] = \text{rank-}1$ {SVD}$(\bcY)$;\\
\nl \textbf{initialize} $c=1,$ $\delta_0 = \norm{\bcY}_F$, $\bcW_0 = \bcW_1 = \lambda_0\bcU_0\bcV_0^\top;$\\
\nl \For{$t=1, \ldots, T$} 
{ 
\nl   $\delta_t = \nu^t \delta_0;$\\
\nl   $\lambda_t = \nu^t (\lambda_0 - \lambda) + \lambda$;\\
\nl   $\theta_t = (c-1) / (c+2)$;\\
\nl   $\bcQ_t = (1 + \theta_t) \bcW_t - \theta_t\bcW_{t-1};$\\
\nl   $\bcZ_t = \nabla\mathscr{L}_{\bcY}(\bcQ_t))$;\\
\nl   $\bcV_{t-1} = \bcV_{t-1} - \bcV_{t}(\bcV_{t}^\top \bcV_{t-1});$\\
\nl   $\bcR_t = \text{QR}([\bcV_t, \bcV_{t-1}])$;\\
\nl   $[\bcU_{t+1}, \bSigma_{t+1}, \bcV_{t+1}] = \texttt{Approx-SVT}(\bcZ_t, \bcR_t, \lambda_t, \delta_t)$;\\
\nl   \If {$\mathscr{F}_{\lambda}(\bcU_{t+1}\bSigma_{t+1}\bcV_{t+1}^\top) > \mathscr{F}_{\lambda}(\bcU_{t}\bSigma_{t}\bcV_{t}^\top)$}
    {
      $c=1;$
    }
\nl   \Else 
    {
      $c=c+1;$
    }
\nl   \If {$|\mathscr{F}_{\lambda}(\bcU_{t+1}\bSigma_{t+1}\bcV_{t+1}^\top) - \mathscr{F}_{\lambda}(\bcU_{t}\bSigma_{t}\bcV_{t}^\top)|\leq \varepsilon$}
    {
    break;
    } 
}
\nl \Return{$\bcW_{T+1}$.}
\end{algorithm}

\subsection{Synthetic datasets}

\paragraph{Software.} The implementation of Algorithm~\ref{algorithm-aisoftimpute} for the nuclear norm penalized estimator~\eqref{def-estimator} was done in MATLAB R2017b on a desktop computer with macOS system, Intel i7 Core 3.5 GHz CPU and 16GB of RAM.
For fast computation of SVD and sparse matrix computations, the experiments call an external package called PROPACK~\citep{Larsen98lanczosbidiagonalization} implemented in C and Fortran. 
The code that generates all figures given below is available from \url{https://github.com/mzalaya/collectivemc}.

\paragraph{Experimental setup.}

In our experiments we focus on square matrices. We set the number of the source matrices $V=3$, then, for each $v\in \{1, 2, 3\}$, the low-rank ground truth  parameter matrices $\bM^v \in \R^{d \times d_v}$ are created with sizes $d \in \{3000, 6000, 9000\}$ and $d_v \in \{1000, 2000, 3000\}$ (hence $d = D = \sum_{v=1}^3 d_v)$.
Each source matrix  $\bM^v$ is constructed as $\bM^v = \bL^v{\bR^v}^\top$ where $\bL^v \in \R^{d \times r_v}$ and $\bR^v \in \R^{d_v \times r_v}$. This gives a random matrix of rank at most $r_v$. The parameter $r_v$ is set to $\{5, 10, 15\}$.
A fraction of the entries of $\bM^v$ is removed uniformly at random with probability $p\in [0, 1]$. Then, 
the matrices $\bM^v$ are scaled so that $\norm{\bM^v}_\infty = \gamma = 1.$

For $\bM^1$, the elements of $\bL^1$ and $\bR^1$ are sampled i.i.d. from the normal distribution $\mathcal{N}(0.5, 1)$.
For $\bM^2$, the entries of $\bL^2$ and $\bR^2$ are i.i.d. according to Poisson distribution with parameter $0.5$.
Finally, for $\bM^3$, the entries of $\bL^3$ and $\bR^3$ are i.i.d. sampled from Bernoulli distribution with parameter $0.5$.
The collective matrix $\bcM$ is constructed by concatenation of the three sources $\bM^1, \bM^2$ and $\bM^3$, namely $\bcM = (\bM^1, \bM^2, \bM^3)$.
All the details of these experiments are given in Table~\ref{tabledetailsdatasets}.

\begin{table}[htbp]
\centering
\begin{tabular}{clcccc}
\toprule%
\multicolumn{2}{c}{}       &$\bM^1$      & $\bM^2$    & $\bM^3$ & $\bcM$\\
\multicolumn{2}{c}{}       &$(Gaussian)$  & $(Poisson)$  & $(Bernoulli)$  & $(Collective)$\\ 
\otoprule%
\multirow{3}*{\texttt{exp.1}}  
&$dimension$                  & $3000\times1000$       & $3000\times1000$ & $3000\times1000$ & $3000\times3000$\\ 
\cmidrule(l){2-6}             
\cmidrule(l){2-6}
& $rank$                        & $5$       & $5$ & $5$ & $unknown$\\   
\cmidrule(l){2-6}
\midrule

\multirow{3}*{\texttt{exp.2}} 
& $dimension$                  & $6000\times2000$      &$6000\times2000$   & $6000\times2000$  & $6000\times6000$  \\ 
\cmidrule(l){2-6}
& $rank$                   & $10$        & $10$ & $10$ & $unknown$\\ 
\cmidrule(l){2-6}
\midrule

\multirow{3}*{\texttt{exp.3}} 
& $dimension$                  & $9000\times3000$      &$9000\times3000$   & $9000\times3000$  & $9000\times9000$  \\ 
\cmidrule(l){2-6}
& $rank$                   & $15$       & $15$ & 15 & $unknown$\\ 
\cmidrule(l){2-6}

\bottomrule 
\end{tabular}
\captionsetup{labelformat=default, skip=10pt}
\caption{Details of the synthetic data in the three experiments.}\label{tabledetailsdatasets}
\end{table}

The details of our experiments are summarized in Figures~\ref{fig:objectives} and~\ref{fig:learning-ranks}. In Figure~\ref{fig:objectives}, we plot the convergence of the objective function $\mathscr{F}_{\lambda}$ versus time in the three experiments.
Note that PLAIS-Impute inherits the speed of AIS-Impute as it does not require performing SVD and it has both low per-iteration cost and fast convergence rate.
In Figure~\ref{fig:objectives}, we plot also the convergence of the objective function $\mathscr{F}_{\lambda}$ versus $-\log(\lambda)$ in the three experiments.
The regularization parameter in the PLAIS-Impute is initialized to a large value and decreased gradually.
In Figure~\ref{fig:learning-ranks}, we illustrate a learning rank curve obtained by PLAIS-Impute, where the green color corresponds to the input rank and the cyan color to the recovered rank of the collective matrix $\bcM$.

\begin{figure}[htbp]%
\centering
\includegraphics[width=\linewidth=0.90\textwidth]{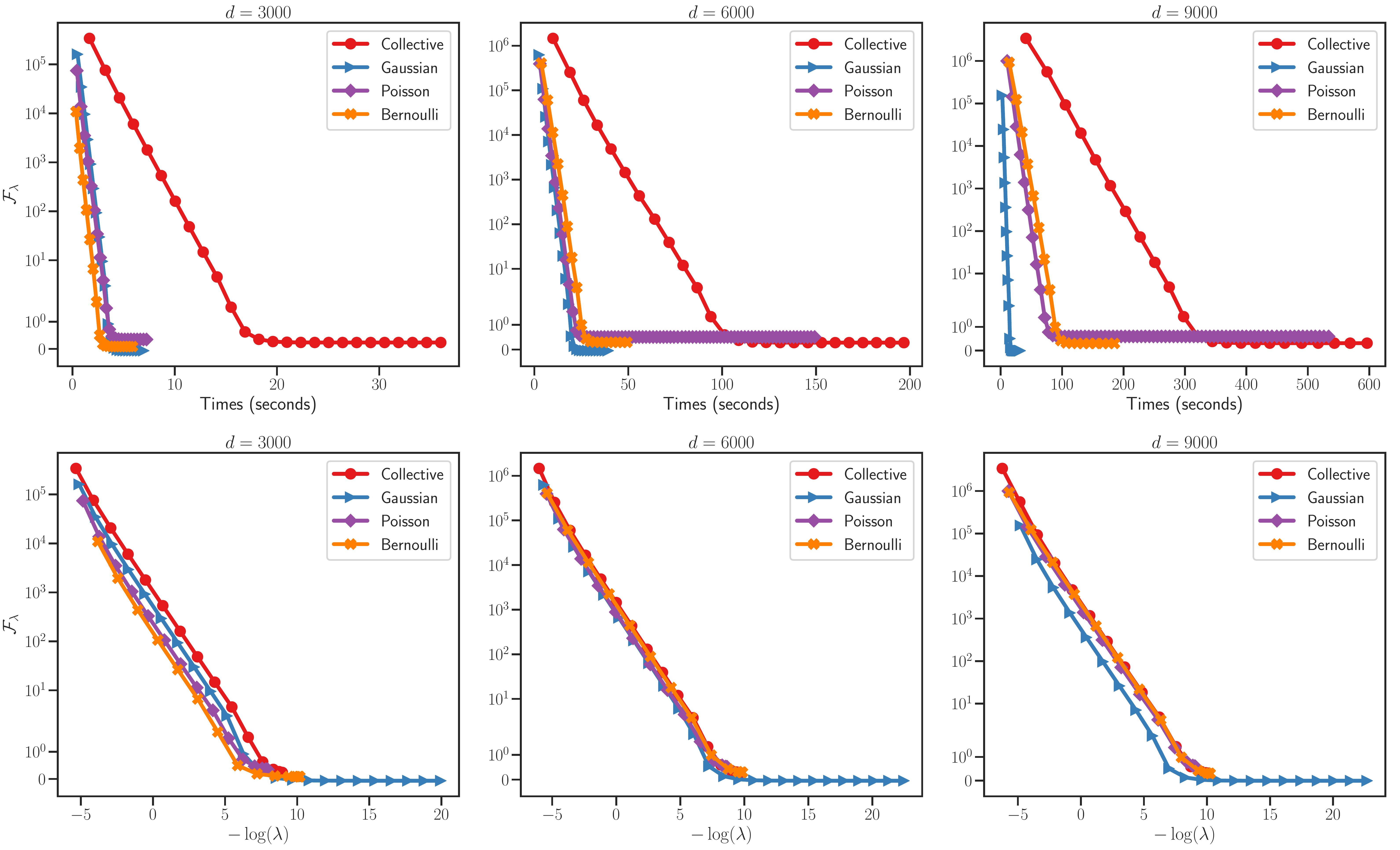}
\caption{Convergence of the objective function $\mathscr{F}_{\lambda}$ in problem~\eqref{def-estimator} versus time (top) and versus $-\log(\lambda)$ (bottom) in the three experiments with $p=0.6$; left for \texttt{exp.1}; middle for \texttt{exp.2}; right for \texttt{exp.3}.
Note that the objective functions for Gaussian, Poisson and Bernoulli distributions are calculated separately by the algorithm.}
\label{fig:objectives}
\end{figure}

\begin{figure}[htbp]%
\centering
\includegraphics[width=\linewidth=0.90\textwidth]{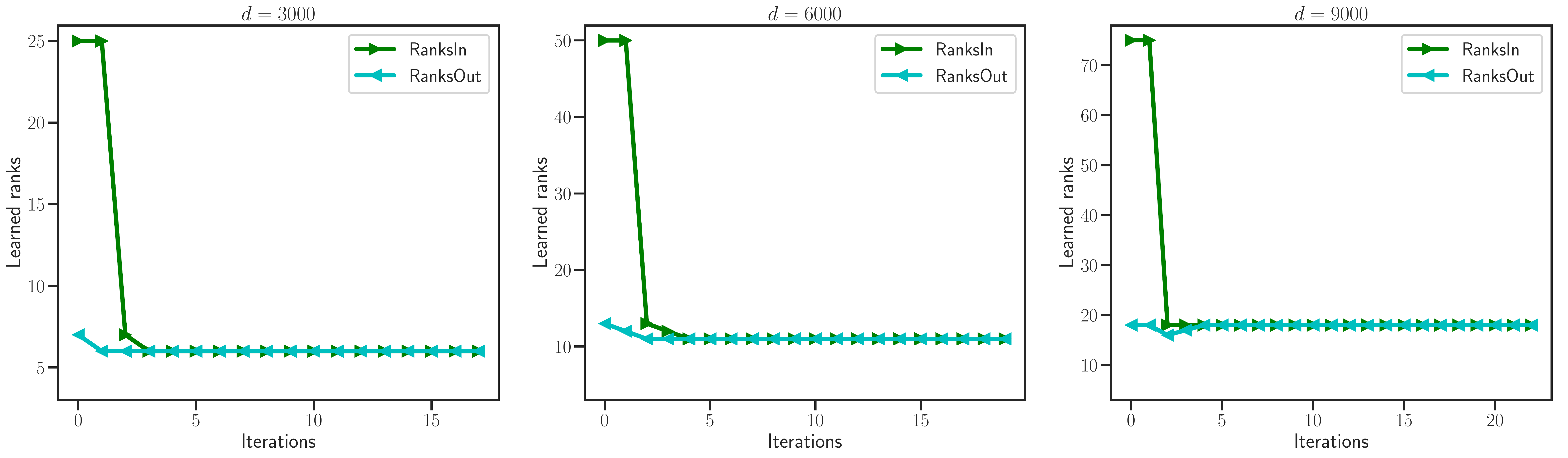}
\caption{Learning ranks curve versus iterations in the three experiments with $p=0.6$; left for \texttt{exp.1}; middle for \texttt{exp.2}; right for \texttt{exp.3}. We initialize the algorithm by setting a rank $r_0 = 5 r$ where $r \in \{5, 10, 15\}.$
The green color corresponds to the input rank while the cyan to the recovered rank of the collective matrix at each iteration.
As can be seen, the two ranks gradually converge to the final recovered rank.}
\label{fig:learning-ranks}
\end{figure}

\paragraph{Evaluation.} 

In our experiments, the PLAIS-Impute algorithm terminates when the absolute difference in the cost function values between two consecutive iterations is less than $\epsilon= 10^{-6}.$
We set the regularization parameter  $\lambda \propto \norm{\nabla\mathscr{L}_{\bcY}(\bcM)}$ as given by Theorem~\ref{theorem1}.  
Note that in step 12 of PLAIS-Impute, the threshold in SVT is given by $\lambda_t$ (defined in step 6), which is decreasing from one iteration to another. This allows to tune the first regularization parameter $\lambda$ in the program~\eqref{def-estimator}. 
We randomly sample $80\%$ of the observed entries for training, and the rest for testing.

In order to measure the the accuracy of our estimator, we employ the relative error (as, e.g., in~\cite{can2010SVTAlgo,davenport14,cai2013jmlr-onebitmaxnorm}) which is widely used metric in matrix completion and is defined by
\begin{equation*}
	\text{RE}(\widehat{\bW}, \bW) = \frac{\norm{\widehat{\bW} - \bW^{o}}_F}{\norm{\bW^{o}}_F},
\end{equation*}
where $\widehat{\bW}$ is the recovered matrix and $\bW^{o}$ is the original full data matrix.

We run the PLAIS-Impute algorithm in each experiment by varying the percentage of known entries $p$ from $0$ to $1$.
In Figure~\ref{fig:relative_errors}, we plot the relative errors as a functions of $p$.
We observe in Figure~\ref{fig:relative_errors} that the relative errors are decaying with $p$.
Note that for each $v\in \{1, 2, 3\}$, the estimator $\widehat{\bM^v}$ is calculated separately using the same program~\eqref{def-estimator}.
The results shown in Figure~\ref{fig:relative_errors} confirm that collective matrix completion approach outperforms the approach that consists in estimating each component source independently.

\begin{figure}[htbp]%
\centering
\includegraphics[width=\linewidth=0.90\textwidth]{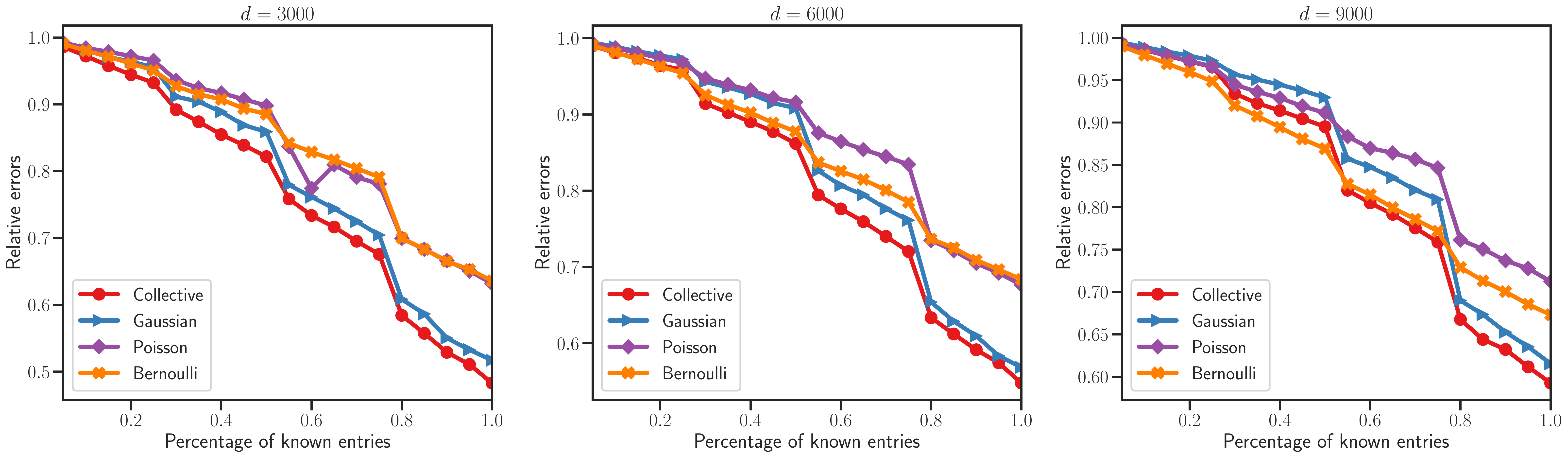}
\caption{Performance on the synthetic data in terms of relative errors between the target and the estimator matrices as a function of the percentage of known entries $p$ from $0$ to $1$.}\label{fig:relative_errors}
\end{figure}

\paragraph{Cold-start problem.} 

To simulate cold-start scenarios,  we choose one of the source matrices $\bM^v$ to be ``cold'' by increasing its sparsity. More precisely, we proceed in the following way: we extract vector of known entries of the chosen matrix and we set the first $1/5$ fraction of its entries to be equal to $0.$ We denote the obtained matrix by $\bM^v_{\text{cold}}$ and the collective matrix by $\bcM^v_\text{cold}$. In $\texttt{exp.1}$, $\texttt{exp.2}$ and $\texttt{exp.3}$, we increase the sparsity of $\bM^1$, $\bM^2$, and $\bM^3$, respectively. Hence, we get the ``cold'' collective matrices $\bcM^1_\text{cold} = (\bM_{\text{cold}}^1, \bM^2, \bM^3)$, $\bcM^2_\text{cold} = (\bM^1, \bM_{\text{cold}}^2, \bM^3)$, and $\bcM^3_{\text{cold}} = (\bM^1, \bM^2, \bM_{\text{cold}}^3)$.

We run $10$ times the PLAIS-Impute algorithm for recovering the source $\bM^v_{\text{cold}}$ and the collective $\bcM^v_{\text{cold}}$ for each $v=1, 2, 3$.
We denote by $\widehat{\bM^v_{\text{comp}}}$ the estimator of $\bM^v_{\text{cold}}$ obtained by running the PLAIS-Impute algorithm only for this component. 
Analogously, we denote $\widehat{\bM^v_{\text{collect}}}$ the estimator of $\bM^v_\text{cold}$ obtained by extracting the $v$-th source of the collective estimator $\widehat{\bcM^v_{\text{cold}}}$.

In Figure~\ref{fig:relative_errors_cold_star}, we report the relative errors $\text{RE}(\widehat{\bM^v_{\text{comp}}},\bM^v_{\text{cold}})$ and $\text{RE}(\widehat{\bM^v_{\text{collect}}},\bM^v_{\text{cold}})$ in the three experiments.
We see that, the collective matrix completion approach compensates the lack of informations in the ``cold'' source matrix.
Therefore, this shared structure among the sources is useful to get better predictions.

\begin{figure}[htbp]%
\centering 
\includegraphics[width=\linewidth=0.90\textwidth]{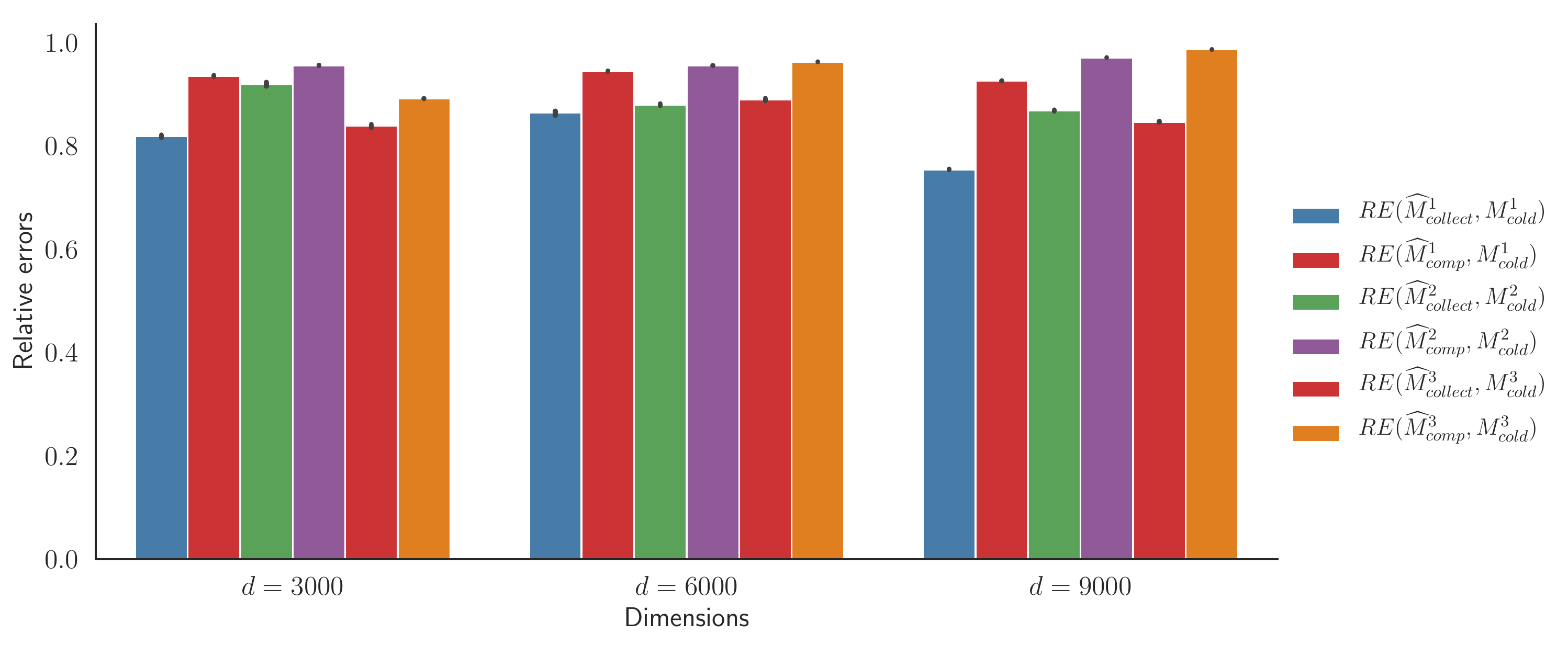}\\
\caption{Relative errors over a set of $10$ randomly generated datasets according to the cold-start scenarios
(with the black lines representing $\pm$ the standard deviation) between the target and the estimator matrices.}
\label{fig:relative_errors_cold_star}
\end{figure}

\section{Conclusion}

This paper studies the problem of recovering a low-rank matrix when the data are collected from multiple and heterogeneous source matrices.
We first consider the setting where, for each source, the matrix entries are sampled from an exponential family distribution.
We then relax this assumption. 
The proposed estimators are based on minimizing  the sum of a goodness-of-fit term and the nuclear norm penalization of the whole collective matrix.
Allowing for non-uniform sampling, we establish upper bounds on the prediction risk of our estimator.
As a by-product of our results, we provide exact minimax optimal rate of convergence for $1$-bit matrix completion which previously was known upto a logarithmic factor.
We present the proximal algorithm PLAIS-Impute to solve the corresponding convex programs.
The empirical study provides evidence of the efficiency of the collective matrix completion approach in the case of joint low-rank structure compared to estimate each source matrices separately.

\section*{Acknowledgments}

 We would like to thank the Associated Editor and the two anonymous Referees for extremely valuable comments and remarks that helped us greatly to improve the paper.  This work was supported by grants from DIM Math Innov Région Ile-de-France \url{https://www.dim-mathinnov.fr} 

\newpage
\appendix

\begin{appendices}
\numberwithin{equation}{section}

\section{Proofs}
\label{proofs}

We provide proofs of the main results, Theorems~\ref{theorem1} and~\ref{theorem-oracle-ranking}, in this section.
The proofs of a few technical lemmas including Lemmas~\ref{lemma-ctrl-sigmaR}, ~\ref{lemma-ctrl-bSigma} and~\ref{lemma-upper-bound-bSigmastar} are also given.
Before that, we recall some basic facts about matrices.

\paragraph{Basic facts about matrices.}

The singular value decomposition (SVD) of $\bA$ has the form $\bA = \sum_{l=1}^{\rk(\bA)}\sigma_l(\bA)u_l(\bA)v_l^\top(\bA)$ with orthonormal vectors $u_1(\bA), \ldots, u_{\rk(\bA)}(\bA)$, orthonormal vectors $v_1(\bA), \ldots, v_{\rk(\bA)}(\bA)$, and real numbers $\sigma_1(\bA) \geq \cdots \geq \sigma_{\rk(\bA)}(\bA)> 0$ (the singular values of $\bA$).
Let $(\mathcal{S}_1(\bA),\mathcal{S}_2(\bA))$ be the pair of linear vectors spaces, where $\mathcal{S}_1(\bA)$ is the linear span space of $\{u_1(\bA), \ldots, u_{\rk(\bA)}(\bA)\}$, and $\mathcal{S}_2(\bcA)$ is the linear span space of $\{v_1(\bA), \ldots, v_{\rk(\bA)}(\bA)\}$.
We denote by $\mathcal{S}_j^\perp(\bA)$ the orthogonal complements of $\mathcal{S}_j(\bA)$, for $j=1, 2$ and by $P_\mathcal{S}$ the projector on the linear subspace $\mathcal{S}$ of $\R^{n}$ or $\R^{m}$.

For two matrices $\bA$ and $\bB$, we set $\mathscr{P}^\perp_{\bA}(\bB) = P_{\mathcal{S}_1^\perp(\bA)}\bB P_{\mathcal{S}_2^\perp(\bA)}$ and $\mathscr{P}_{\bA} (\bB) = \bB - \mathscr{P}^\perp_{\bA}(\bB)$.
Since $\mathscr{P}_{\bA} (\bB) = P_{\mathcal{S}_1(\bA)}\bB + P_{\mathcal{S}_1^\perp(\bA)}\bB P_{\mathcal{S}_2(\bA)}$, and $\rk(P_{\mathcal{S}_j(\bA)}\bB) \leq \rk(\bA)$, we have that
\begin{equation}
\label{fact-1-matrix}
\rk(\mathscr{P}_{\bA} (\bB)) \leq 2 \rk(\bA).
\end{equation}
It is easy to see that for two matrices $\bA$ and $\bB$~\citep{klopp2014}
\begin{equation}
\label{klopp-thm3}
\norm{\bA}_* - \norm{\bB}_* \leq \norm{\mathscr{P}_{\bA}(\bA - \bB)}_* - \norm{\mathscr{P}^\perp_{\bA}(\bA - \bB)}_*.
\end{equation}
Finally, we recall the well-known trace duality property: for all $\bA, \bB \in \R^{n \times m}$, we have 
\begin{equation*}
\label{trace duality}
|\inr{\bA, \bB}| \leq \norm{\bB}\norm{\bA}_*.
\end{equation*}

\subsection{Proof of Theorem~\ref{theorem1}}
\label{proof-theorem1}

First, noting that $\widehat{\bcM}$ is optimal and $\bcM$ is feasible for the convex optimization problem~\eqref{def-estimator}, we thus have the basic inequality that   
\begin{align*}
\frac{1}{d_uD}\sum_{v\in [V]}\sum_{(i,j)\in [d_u]\times[d_v]} &B_{ij}^v\big(G^v(\hat{M}^v_{ij}) - Y^v_{ij}\hat{M}^v_{ij}\big) + \lambda\norm{\widehat{\bcM}}_* \\
&\qquad \qquad \leq \frac{1}{d_uD}\sum_{v\in [V]}\sum_{(i,j)\in [d_u]\times[d_v]} B_{ij}^v\big(G^v({M}^v_{ij}) - Y^v_{ij}{M}^v_{ij}\big) + \lambda\norm{{\bcM}}_*.
\end{align*}
It yields
\begin{equation*}
\frac{1}{d_uD}\sum_{v\in [V]}\sum_{(i,j)\in [d_u]\times[d_v]} B_{ij}^v\Big(\big(G^v(\hat{M}^v_{ij})  - G^v({M}^v_{ij})\big) - Y^v_{ij}\big(\hat{M}^v_{ij} - {M}^v_{ij}\big)\Big) \leq \lambda(\norm{{\bcM}}_* - \norm{\widehat{\bcM}}_*).
\end{equation*}
Using the Bregman divergence associated to each $G^v$, we get
\begin{align*}
\frac{1}{d_uD}\sum_{v\in [V]}&\sum_{(i,j)\in [d_u]\times[d_v]}B_{ij}^v d_{G^v}(\hat{M}^v_{ij}, {M}^v_{ij})\\
&\leq \lambda(\norm{{\bcM}}_* - \norm{\widehat{\bcM}}_*) - \frac{1}{d_uD}\sum_{v\in [V]}\sum_{(i,j)\in [d_u]\times[d_v]} B_{ij}^v\big((G^v)'({M}^v_{ij}) - Y^v_{ij} \big)\big(\hat{M}^v_{ij} - {M}^v_{ij}\big).
\end{align*}
Therefore, using the duality between $\norm{\cdot}_*$ and $\norm{\cdot}$, we arrive at 
\begin{align*}
\frac{1}{d_uD}\sum_{v\in [V]}\sum_{(i,j)\in [d_u]\times[d_v]} B_{ij}^v d_{G^v}(\hat{M}^v_{ij}, {M}^v_{ij})
&\leq \lambda(\norm{{\bcM}}_* - \norm{\widehat{\bcM}}_*) - \inr{\nabla\mathscr{L}_{\bcY}(\bcM), \widehat{\bcM} - \bcM}\\
&\leq \lambda(\norm{{\bcM}}_* - \norm{\widehat{\bcM}}_*) + \norm{\nabla\mathscr{L}_{\bcY}(\bcM)} \norm{\widehat{\bcM} - \bcM}_*.
\end{align*}
Besides, using the assumption $\lambda \geq  2\norm{\nabla\mathscr{L}_{\bcY}(\bcM)}$ and inequality~\eqref{klopp-thm3} lead to 
\begin{equation*}
\frac{1}{d_uD}\sum_{v\in [V]}\sum_{(i,j)\in [d_u]\times[d_v]} B_{ij}^v d_{G^v}(\hat{M}^v_{ij}, {M}^v_{ij}) \leq \frac{3\lambda}{2}\norm{\mathscr{P}_{\bcM}\big(\widehat{\bcM} - \bcM\big)}_*.
\end{equation*}
Since $\norm{\mathscr{P}_{\bcA}(\bcB)}_* \leq \sqrt{2\rk(\bcA)} \norm{\bcB}_F$ for any two matrices $\bcA$ and $\bcB$, we obtain
\begin{equation}
\label{important-inequ}
\frac{1}{d_uD}\sum_{v\in [V]}\sum_{(i,j)\in [d_u]\times[d_v]} B_{ij}^vd_{G^v}(\hat{M}^v_{ij}, {M}^v_{ij}) \leq \frac{3\lambda}{2}\sqrt{2\rk(\bcM)}\norm{\widehat{\bcM} - \bcM}_F.
\end{equation}
Now, Assumption~\ref{assump-Gv-bound} implies that the Bregman divergence satisfies $L^2_\gamma(x - y)^2 \leq 2 d_G^{v}(x, y) \leq U^2_\gamma(x - y)^2,$ then we get 
\begin{align}
\label{second-important-inequ}
\Delta^2_{\bcY}(\widehat{\bcM}, \bcM)
&\leq \frac{2}{L^2_\gamma}\frac{1}{d_uD}\sum_{v\in [V]}\sum_{(i,j)\in [d_u]\times[d_v]} B_{ij}^vd_{G^v}(\hat{M}^v_{ij}, {M}^v_{ij}),
\end{align}
where 
\begin{equation*}
\Delta^2_{\bcY}(\widehat{\bcM}, \bcM) = \frac{1}{d_uD}\sum_{v\in [V]}\sum_{(i,j)\in [d_u]\times[d_v]} B_{ij}^v(\hat{M}^v_{ij} - {M}^v_{ij})^2.
\end{equation*}
Combining~\eqref{important-inequ} and~\eqref{second-important-inequ}, we arrive at 
\begin{equation}
\label{Delta-Frob-norm}
\Delta^2_{\bcY}(\widehat{\bcM}, \bcM) \leq \frac{3\lambda}{L_\gamma^2} \sqrt{2\rk(\bcM)}\norm{\widehat{\bcM} - \bcM}_F.
\end{equation}

Let us now define the threshold $\beta = \frac{946\gamma^2\log (d_u + D)}{pd_uD }$ and distinguish the two following cases that allows us to obtain an upper bound for the estimation error:\\
\noindent{{\it Case 1:\,\,}} 
if $(d_uD)^{-1}\norm{\widehat{\bcM} - \bcM}_{\Pi, F}^2 < \beta$, then the statement of Theorem~\ref{theorem1} is true.\\
\noindent{{\it Case 2:\,\,}} it remains to consider the case $(d_uD)^{-1}\norm{\widehat{\bcM} - \bcM}_{\Pi, F}^2 \geq \beta$.
Lemma~\ref{lemma-crtl-nucl-2} in Appendix~\ref{appendix-useful-lemmas} implies $\norm{\widehat{\bcM} - \bcM}_F \geq \frac{1}{4 \sqrt{2\rk(\bcM)}}\norm{\widehat{\bcM} - \bcM}_*$, then we obtain
\begin{align*}
\label{hatbcm-belongs-tpC}
\norm{\widehat{\bcM}- \bcM}_* \leq \sqrt{32\rk(\bcM)}\norm{\widehat{\bcM} - \bcM}_{F}.
\end{align*}
This leads to $\widehat{\bcM} \in \mathscr{C}\big(\beta, 32{}\rk(\bcM)\big),$ where the set 
\begin{align}
\mathscr{C}(\beta, r) = \bigg\{\bcW \in \mathscr{C}_\infty(\gamma):  &\norm{\bcM - \bcW}_* \leq \sqrt{r}\norm{{\bcW} - \bcM}_{F}\text{ and } (d_uD)^{-1}\norm{{\bcW} - \bcM}_{\Pi,F}^2 \geq \beta\bigg\}.
\end{align}
Using Lemma~\ref{lemma-prob-upper-bound} in Appendix~\ref{appendix-useful-lemmas}, we have
\begin{equation}
\label{thme1-import-ineq}
\Delta^2_{\bcY}(\widehat{\bcM}, \bcM) \geq \frac{\norm{{\bcW} - \bcM}_{\Pi,F}^2}{2d_uD}
- {44536\rk(\bcM)\gamma^2}(\E[\norm{\bSigma_R}])^2 - \frac{5567\gamma^2}{{d_uD}p}.
\end{equation}
Together~\eqref{thme1-import-ineq} and~\eqref{Delta-Frob-norm} imply 
\begin{align*}
\frac{1}{2d_uD}\norm{\widehat{\bcM} - \bcM}^2_{\Pi,F} &\leq \frac{3\lambda}{L^2_\gamma} \sqrt{2\rk(\bcM)}\norm{\widehat{\bcM} - \bcM}_F\\
&\qquad  + {44536\rk(\bcM)\gamma^2}(\E[\norm{\bSigma_R}])^2 + \frac{5567\gamma^2}{{pd_uD}}\\
& \leq \frac{18\lambda^2d_uD}{pL^4_\gamma} {\rk(\bcM)} + \frac{1}{4d_uD}\norm{\widehat{\bcM} - \bcM}^2_{\Pi,F}\\
&\qquad  + {44536\rk(\bcM)\gamma^2}(\E[\norm{\bSigma_R}])^2 + \frac{5567\gamma^2}{{pd_uD}}.
\end{align*}
Then, 
\begin{align*}
\frac{1}{4d_uD}\norm{\widehat{\bcM} - \bcM}^2_{\Pi,F} &\leq \frac{18\lambda^2d_uD}{pL^4_\gamma} {\rk(\bcM)} \\
&\qquad  + {44536p^{-1}d_uD\rk(\bcM)\gamma^2}(\E[\norm{\bSigma_R}])^2 +\frac{5567\gamma^2}{{d_uD}p},
\end{align*}
and,
\begin{align*}
\frac{1}{d_uD}\norm{\widehat{\bcM} - \bcM}^2_{\Pi,F}
\leq p^{-1}\max\bigg(d_uD\rk(\bcM)\bigg(\frac{c_1\lambda^2}{L^4_\gamma} + c_2\gamma^2(\E[\norm{\bSigma_R}])^2\bigg), \frac{c_3\gamma^2}{{d_uD}}\bigg),
\end{align*}
where $c_1,c_2$ and $c_3$ are numerical constants.
This concludes the proof of Theorem~\ref{theorem1}.

\subsection{Proof of Lemma~\ref{lemma-ctrl-sigmaR}}
\label{proof-lemma-ctrl-sigmaR}

We use the following result:
\begin{proposition}(Corollary 3.3 in~\cite{bandeira2016})
\label{proposition-ctrl-expect-spectral-norm}
Let $\bW$ be the $n \times m$ rectangular matrix whose entries $W_{ij}$ are independent centered bounded random variables.
Then there exists a universal constant $\cst$ such that 
\begin{equation*}
\E[\norm{\bW}] \leq \cst\Big(\kappa_1 \vee \kappa_2 + \kappa_*\sqrt{\log(n\wedge m)}\Big),
\end{equation*}
where we have defined
\begin{align*}
\kappa_1 = \max_{i\in[n]} \sqrt{\sum_{j\in[m]} \E[W_{i,j}^2]},\quad 
\kappa_2 = \max_{j\in[m]} \sqrt{\sum_{i\in[n]} \E[W_{i,j}^2]}, \quad \text{ and }\quad  
\kappa_* = \max_{(i,j)\in[n]\times[m]}|W_{ij}|.
\end{align*}
\end{proposition}

We apply Proposition~\ref{proposition-ctrl-expect-spectral-norm} to $\bSigma_R = \frac{1}{d_uD}\sum_{v\in [V]}\sum_{(i,j) \in [d_u]\times[d_v]}\varepsilon^v_{ij}B_{ij}^vE^v_{ij}$.
We compute
\begin{align*}
\kappa_1 = \frac{1}{d_uD}\max_{i\in[d_u]} \sqrt{\sum_{v\in[V]}\sum_{j\in[d_v]} \E[(\varepsilon^v_{ij})^2(B_{ij}^v)^2]} &= \frac{1}{d_uD}\max_{i\in[d_u]} \sqrt{\sum_{v\in[V]}\sum_{j\in[d_v]} \pi^v_{ij}}\\
&= \frac{1}{d_uD}\max_{i\in[d_u]} \sqrt{\pi_{i\bcdot}},
\end{align*}
\begin{align*}
\kappa_2 = \frac{1}{d_uD}\max_{v\in[V]}\max_{j\in[d_v]} \sqrt{\sum_{i\in[d_u]} \E[(\varepsilon^v_{ij})^2(B_{ij}^v)^2]} &=\frac{1}{d_uD} \max_{v\in[V]}\max_{j\in[d_v]} \sqrt{\sum_{i\in[d_u]} \pi^v_{ij}}\\
&\leq \frac{1}{d_uD}\max_{j\in[d_v]} \sqrt{\max_{v\in[V]}\sum_{i\in[d_u]} \pi^v_{ij}}\\
& \leq \frac{1}{d_uD}\max_{j\in[d_v]} \sqrt{\pi_{\bcdot j}},
\end{align*}
and $\kappa_* = \frac{1}{d_uD}\max_{v\in[V]}\max_{(i,j)\in[d_u]\times[d_v]} |\varepsilon^v_{ij}B_{ij}| \leq \frac{1}{d_uD}.$
Using inequality~\eqref{upper-bound-marginal}, we have $\kappa_1 \leq \frac{\sqrt{\mu}}{d_uD}$ and $\kappa_2 \leq \frac{\sqrt{\mu}}{d_uD}$.
Then, $\kappa_1 \vee \kappa_2 \leq \frac{\sqrt{\mu}}{d_uD}$, which establishes Lemma~\ref{lemma-ctrl-sigmaR}.

\subsection{Proof of Lemma~\ref{lemma-lemma-ctrl-bSigma}} 
\label{proof-lemma-ctrl-bSigma}

We  write $\nabla\mathscr{L}_{\bcY}(\bcM) = -\frac{1}{d_uD}\sum_{v\in [V]}\sum_{(i,j) \in [d_u]\times[d_v]} H^v_{ij}E^v_{ij}$, with $H^v_{ij} = B_{ij}^v\big(X^v_{ij} - (G^v)'(M^v_{ij})\big)$. 
For a truncation level $T > 0$ to be chosen, we decompose $\nabla\mathscr{L}_{\bcY}(\bcM)= \bSigma_1 + \bSigma_2$, where
\begin{equation*}
\bSigma_1 = -\frac{1}{d_uD}\sum_{v\in [V]}\sum_{(i,j) \in [d_u]\times[d_v]} \big(H^v_{ij}\ind{}{_{((X^v_{ij} - \E[X^v_{ij}])\leq T)}} - \E\big[H^v_{ij}\ind{}{_{((X^v_{ij} - \E[X^v_{ij}])\leq T)}}\big]\big)E^v_{ij},
\end{equation*}
and
\begin{equation*}
\bSigma_2 = -\frac{1}{d_uD}\sum_{v\in [V]}\sum_{(i,j) \in [d_u]\times[d_v]} \big(H^v_{ij}\ind{}{_{((X^v_{ij} - \E[X^v_{ij}]) > T)}} - \E\big[H^v_{ij}\ind{}{_{((X^v_{ij} - \E[X^v_{ij}]) > T)}}\big]\big)E^v_{ij},
\end{equation*}
then, the triangular inequality implies $\norm{\nabla\mathscr{L}_{\bcY}(\bcM)} \leq \norm{\bSigma_1} + \norm{\bSigma_2}.$
Then, the proof is divided on two steps:

\noindent{{\it Step 1: control of $\norm{\bSigma_1}$.\,\,}}
In order to control $\norm{\bSigma_1}$, we use the following bound on the spectral norms of random matrices.
It is obtained by extension to rectangular matrices via self-adjoint dilation of Corollary 3.12 and Remark 3.13 in \cite{bandeira2016}.
\begin{proposition}\citep{bandeira2016}
\label{proposition-ctrl-spect-norm}
Let $\bW$ be the $n \times m$ rectangular matrix whose entries $W_{ij}$ are independent centered bounded random variables.
Then, for any $0\leq \epsilon \leq 1/2$ there exists a universal constant $\cst_\epsilon$ such that for every $x \geq 0$,
\begin{equation*}
\P\big[\norm{\bW} \geq 2\sqrt{2}(1+\epsilon)(\kappa_1 \vee \kappa_2) + x\big] \leq (n \wedge m) \exp\Big(-\frac{x^2}{\cst_\epsilon \kappa^2_*}\Big),
\end{equation*}
where $\kappa_1$, $\kappa_2$, and $\kappa_*$ are defined as in Proposition~\ref{proposition-ctrl-expect-spectral-norm}.
\end{proposition}
We apply Proposition~\ref{proposition-ctrl-spect-norm} to $\bSigma_1$.
We compute 
\begin{equation*}
\kappa_1 = \frac{1}{d_uD}\max_{i\in[d_u]} \sqrt{\sum_{v\in[V]}\sum_{j\in[d_v]} \E\Big[\big(H^v_{ij}\ind{}{_{((X^v_{ij} - \E[X^v_{ij}])\leq T)}} - \E\big[H^v_{ij}\ind{}{_{((X^v_{ij} - \E[X^v_{ij}])\leq T)}}\big]\big)^2\Big]}.
\end{equation*}
Besides, we have 
\begin{align*}
\E\Big[\big(H^v_{ij}\ind{}{_{((X^v_{ij} - \E[X^v_{ij}])\leq T)}} - \E\big[H^v_{ij}\ind{}{_{((X^v_{ij} - \E[X^v_{ij}])\leq T)}}\big]\big)^2\Big] \leq \E\big[(H^v_{ij})^2\ind{}{_{((X^v_{ij} - \E[X^v_{ij}]) \leq T)}}\big],
\end{align*}
and
\begin{align*}
\E\big[(H^v_{ij})^2\ind{}{_{((X^v_{ij} - \E[X^v_{ij}]) \leq T)}}\big]
&= \E\big[(B_{ij}^v)^2\big(X^v_{ij} - \E[X_{ij}^v])^2\ind{}{_{((X^v_{ij} - \E[X^v_{ij}]) \leq T)}}\big]\\
&\leq \pi^v_{ij} \V ar[X^v_{ij}]\\
&= \pi^v_{ij} (G^v)''(M^v_{ij}).
\end{align*}
By Assumption~\ref{assump-Gv-bound}, we obtain $\E\big[(H^v_{ij})^2\ind{}{_{((X^v_{ij} - \E[X^v_{ij}]) \leq T)}}\big] \leq\pi^v_{ij}U_\gamma^2$ for all $v\in[V], (i,j)\in[d_u]\times [d_v]$.
Then,
\begin{align*}
\kappa_1 \leq \frac{{U_\gamma}}{d_uD}\max_{i\in[d_u]} \sqrt{\sum_{v\in[V]}\sum_{j\in[d_v]}\pi^v_{ij}} \leq \frac{{U_\gamma}}{d_uD}\max_{i\in[d_u]} \sqrt{\pi^v_{i\bcdot}} \leq \frac{U_\gamma\sqrt{\mu}}{d_uD},
\end{align*}
and
\begin{align*}
\kappa_2 & 
\leq \frac{U_\gamma}{d_uD}\max_{j\in[d_v]}\sqrt{\max_{v\in[V]}\sum_{i\in[d_u]}  \pi^v_{ij}}
\leq \frac{U_\gamma}{d_uD}\max_{j\in[d_v]}\sqrt{\pi_{\bcdot j}}
\leq \frac{U_\gamma\sqrt{\mu}}{d_uD}.
\end{align*}
It yields, $\kappa_1 \vee \kappa_2 \leq \frac{U_\gamma\sqrt{\mu}}{d_uD}$.
Moreover, we have $\E\big[H^v_{ij}\ind{}{_{((X^v_{ij} - \E[X^v_{ij}]) \leq T)}}\big] \leq T,$ which entails $\kappa_* \leq \frac{2T}{d_uD}$.
By choosing $\epsilon = 1/2$ in Proposition~\ref{proposition-ctrl-spect-norm}, we obtain, with probability at least $1 - 4 (d_u \wedge D) e^{-x^2},$
\begin{align*}
\norm{\bSigma_1}
&\leq \frac{3U_\gamma\sqrt{2\mu}+ 2\sqrt{\cst_{1/2}}xT}{d_uD}.
\end{align*}
Therefore, by setting $x = \sqrt{2\log(d_u +D)}$, we get with probability at least $1 - 4/(d_u +D)$,
\begin{align}
\label{final-ctrl-bSigma1-expnoise}
\norm{\bSigma_1} 
&\leq \frac{ 3U_\gamma\sqrt{2\mu}+ 2\sqrt{\cst_{1/2}}\sqrt{2\log(d_u +D)}T}{d_uD}.
\end{align}

\noindent{{\it Step 2: control of $\norm{\bSigma_2}$.\,\,}}
To control $\norm{\bSigma_2}$, we use Chebyshev's inequality, that is 
\begin{equation*}
\P\big[\norm{\bSigma_2} \geq \E[\norm{\bSigma_2}]+ x\big] \leq \frac{\V ar[\norm{\bSigma_2}]}{x^2}, \text{ for all } x > 0.
\end{equation*}
We start by estimating $\E[\norm{\bSigma_2}]$.
We use the fact that $\E[\norm{\bSigma_2}] \leq \E[\norm{\bSigma_2}_F]$:
\begin{align*}
\E\big[\norm{\bSigma_2}_F^2\big] &= \frac{1}{(d_uD)^2}\sum_{v\in [V]}\sum_{(i,j) \in [d_u]\times[d_v]} \E\big[\big(H^v_{ij}\ind{}{_{((X^v_{ij} - \E[X^v_{ij}]) > T)}} - \E\big[H^v_{ij}\ind{}{_{((X^v_{ij} - \E[X^v_{ij}]) > T)}}\big]\big)^2\big]\\
&\leq \frac{1}{(d_uD)^2}\sum_{v\in [V]}\sum_{(i,j) \in [d_u]\times[d_v]} \E\big[(H^v_{ij})^2\ind{}{_{((X^v_{ij} - \E[X^v_{ij}]) > T)}}\big]\\
&\leq \frac{1}{(d_uD)^2}\sum_{v\in [V]}\sum_{(i,j) \in [d_u]\times[d_v]} \pi^v_{ij}\E\big[(X^v_{ij} - \E[X^v_{ij}])^2\ind{}{_{((X^v_{ij} - \E[X^v_{ij}]) > T)}}\big]\\
&\leq \frac{1}{(d_uD)^2}\sum_{v\in [V]}\sum_{(i,j) \in [d_u]\times[d_v]} \pi^v_{ij}\sqrt{\E\big[(X^v_{ij} - \E[X^v_{ij}])^4\big]} \sqrt{\P\big[{}{{((X^v_{ij} - \E[X^v_{ij}]) > T)}}\big]}.
\end{align*}
By Lemma~\ref{lemma-subgaussian-tail-X}, we have that  $X_{ij}^v - \E[X_{ij}^v]$ is an $(U_\gamma, K)$-sub-exponential random variable for every $v \in[V]$ and $(i,j) \in [d_u]\times[d_v]$.
It yields, using $(2)$ in Theorem~\ref{lemma-properties-sub}, that 
\begin{align*}
\E\big[(X^v_{ij} - \E[X^v_{ij}])^p\big] \leq \cst {p}^p \norm{X^v_{ij}}_{\psi_1}^p, \text{ for every } p \geq 1,
\end{align*}
and by (1) in Theorem~\ref{lemma-properties-sub} 
\begin{equation*}
\P\big[{}{{|X^v_{ij}- \E[X^v_{ij}]|> T}}\big] \leq \exp\Big(1  - \frac{T}{\cst_{\text{se}} \norm{X^v_{ij}}_{\psi_1}}\Big),
\end{equation*}
where $\cst$ and $\cst_{\text{se}}$ are absolute constants.
Consequently,
\begin{align*}
\E\big[\norm{\bSigma_2}_F^2\big] &\leq \frac{\cst}{(d_uD)^2}\sum_{v\in [V]}\sum_{(i,j) \in [d_u]\times[d_v]} \pi^v_{ij} \sqrt{\norm{X^v_{ij}}_{\psi_1}^4}\sqrt{\exp\Big(1  - \frac{T}{\cst_{\text{se}} \norm{X^v_{ij}}_{\psi_1}}\Big)}\\
&\leq \frac{\cst }{(d_uD)^2}\sum_{v\in [V]}\sum_{(i,j) \in [d_u]\times[d_v]}(U_\gamma \vee K)^2 \pi^v_{ij}\sqrt{\exp\Big( - \frac{T}{\cst_{\text{se}} K}\Big)}. 
\end{align*}
We choose $T = T_* := 4\cst_{\text{se}} (U_\gamma \vee K) \log(d_u\vee D)$. It yields, 
\begin{align*}
\E\big[\norm{\bSigma_2}_F^2\big] &\leq \frac{\cst}{(d_uD)^2} \frac{1}{(d_u\vee D)^2}\sum_{v\in [V]}\sum_{(i,j) \in [d_u]\times[d_v]} (U_\gamma \vee K)^2\pi^v_{ij} \\
&\leq \frac{\cst (U_\gamma \vee K)^2}{(d_uD)^2} \frac{1}{(d_u\vee D)^2}\sum_{v\in [V]}\sum_{j \in[d_v]} \pi^v_{ij} \\
&\leq \frac{\cst (U_\gamma \vee K)^2}{(d_uD)^2} \frac{1}{(d_u\vee D)^2}(d_u\vee D) \mu\\
&\leq \frac{\cst (U_\gamma \vee K)^2\mu}{(d_uD)^2d_u\vee D}.
\end{align*}
Using the fact that $x \mapsto \sqrt{x}$ is concave, we obtain 
\begin{equation} 
\label{crtl-norm-bsigma2-expect}  
\E[\norm{\bSigma_2}] \leq \E[\norm{\bSigma_2}_F] \leq \sqrt{\E\big[\norm{\bSigma_2}_F^2\big]} \leq  \sqrt{\frac{\cst (U_\gamma \vee K)^2\mu}{(d_uD)^2d_u\vee D}}\leq \frac{\cst (U_\gamma \vee K) \sqrt{\mu}}{d_uD \sqrt{d_u \vee D}}.
\end{equation}

Let us now control the variance of $\norm{\bSigma_2}$. 
We have immediately, using~\eqref{crtl-norm-bsigma2-expect},
\begin{align*}
\V ar [\norm{\bSigma_2}] \leq \E[\norm{\bSigma_2}^2] \leq \E\big[\norm{\bSigma_2}_F^2\big] \leq \frac{\cst (U_\gamma \vee K)^2\mu}{(d_uD)^2d_u\vee D}.
\end{align*}
By Chebyshev's inequality and using~\eqref{crtl-norm-bsigma2-expect}, we have,  with probability at least $1 - 4 /(d_u + D)$,
\begin{align}
\label{final-ctrl-bSigma2}
\norm{\bSigma_2} \leq \frac{\cst (U_\gamma \vee K) \sqrt{\mu}}{d_uD \sqrt{d_u \vee D}} + \frac{\cst (U_\gamma \vee K) \sqrt{\mu}}{d_uD} \leq \frac{\cst (U_\gamma \vee K) \sqrt{\mu}}{d_uD}.
\end{align}
Finally, combining~\eqref{final-ctrl-bSigma1-expnoise} and~\eqref{final-ctrl-bSigma2}, we obtain, with probability at least $1 - 4/(d_u + D)$,
\begin{align*}
\norm{\nabla\mathscr{L}_{\bcY}(\bcM)} 
&\leq \frac{3U_\gamma\sqrt{2\mu}+ 8(U_\gamma \vee K)\cst_{\text{se}}\sqrt{2\cst_{1/2}\log(d_u +D)}  \log(d_u\vee D) + \cst (U_\gamma \vee K) \sqrt{\mu}}{d_uD}
\end{align*}
Then, 
\begin{align*}
\norm{\nabla\mathscr{L}_{\bcY}(\bcM)} 
&\leq \cst \bigg(\frac{(U_\gamma \vee K)\big(\sqrt{\mu} + (\log(d_u\vee D))^{3/2}\big)}{d_uD} \bigg),
\end{align*}
where $\cst$ is an absolute constant.
This finishes the proof of Lemma~\ref{lemma-lemma-ctrl-bSigma}.

\subsection{Proof of Theorem~\ref{theorem-oracle-ranking}}
\label{proof-of-theorem-oracle-ranking}

We start the proof with the following inequality using the fact that
$\widehat{\bcM}$ is the minimizer of the objective function in problem~\eqref{def-estimator-Lip-loss}
\begin{equation*}
0 \leq -({R}_{\bcY}(\widehat{\bcM})+ \Lambda \norm{\widehat{\bcM}}_*) + ({R}_{\bcY}(\ddt{\bcM})+ \Lambda \norm{\ddt{\bcM}}_*).
\end{equation*}
Then, by adding $R(\widehat{\bcM}) - R(\ddt{\bcM}) \geq 0$, we obtain
\begin{align*}
R(\widehat{\bcM}) - R(\ddt{\bcM}) 
&\leq -\big\{\big(R_{\bcY}(\widehat{\bcM}) - {R}_{\bcY}(\ddt{\bcM})\big) - \big(R(\widehat{\bcM}) - R(\ddt{\bcM})\big)\big\}+ \Lambda \big(\norm{\ddt{\bcM}}_* - \norm{\widehat{\bcM}}_*\big).
\end{align*}
\eqref{klopp-thm3} implies $\norm{\bA}_* - \norm{\bB}_* \leq \norm{\mathscr{P}_{\bA}(\bA - \bB)}_*$ and we get 
\begin{align}
\label{equation-excess-risk}
R(\widehat{\bcM}) - R(\ddt{\bcM})
&\leq -\big\{\big(R_{\bcY}(\widehat{\bcM}) - {R}_{\bcY}(\ddt{\bcM})\big) - \big(R(\widehat{\bcM}) - R(\ddt{\bcM})\big)\big\} + \Lambda \norm{\mathscr{P}_{\ddt{\bcM}}(\ddt{\bcM} - \widehat{\bcM})}_*\nonumber\\
&\leq -\big(R_{\bcY}(\widehat{\bcM}) - {R}_{\bcY}(\ddt{\bcM})\big) + \big(R(\widehat{\bcM}) - R(\ddt{\bcM})\big)\\
&\qquad + \Lambda \sqrt{2\rk(\ddt{\bcM})}\norm{\widehat{\bcM} - \ddt{\bcM}}_F\nonumber.
\end{align}

Let us now define the threshold $\nu = \frac{32\big(1 + e\sqrt{{3\rho}/{\varsig \gamma}}\big)\rho\gamma\log(d_u+D)}{3pd_uD}$ and distinguish the two following cases that allows us to obtain an upper bound for the prediction error:\\
\noindent{{\it Case 1:\,\,}} if $R(\widehat{\bcM}) - R(\ddt{\bcM})< \nu$, then the statement of Theorem~\ref{theorem-oracle-ranking} is true.\\
\noindent{{\it Case 2:\,\,}} it remains to consider the case $R(\widehat{\bcM}) - R(\ddt{\bcM}) \geq \nu$.
Lemma~\ref{lemma-diff-Empir-loss} implies
\begin{equation*}
\norm{\widehat{\bcM} - \ddt{\bcM}}_* \leq \sqrt{32\rk(\ddt{\bcM})}\norm{\widehat{\bcM} - \ddt{\bcM}}_{F},
\end{equation*}
then $\widehat{\bcM} \in \mathscr{Q}(\nu, 32\rk(\ddt{\bcM}))$ where  
\begin{align*}
\mathscr{Q}(\nu, r) = \bigg\{\bcQ \in \mathscr{C}_\infty(\gamma):  &\norm{\bcQ - \ddt{\bcM}}_* \leq \sqrt{r}\norm{{\bcQ} - \ddt{\bcM}}_{F}\text{ and } R(\bcQ) - R(\ddt{\bcM}) \geq \nu\bigg\}.
\end{align*}
Using Lemma~\ref{lemma-prob-upper-bound-loss-case}, we have 
\begin{align}
\label{lemma-prob-upper-oracle-ranking}
&R(\widehat{\bcM}) - R(\ddt{\bcM}) -\big(R_{\bcY}(\widehat{\bcM}) - {R}_{\bcY}(\ddt{\bcM})\big) \nonumber\\
&\qquad \leq \frac{R({\widehat{\bcM}}) - R(\ddt{\bcM})}{2} + \frac{\cst \rk(\ddt{\bcM})\rho^2\varsig^{-1}(\E[\norm{\bSigma_R}])^2}{(1/4e) + (1 - 1/\sqrt{4e})\sqrt{{3\rho}/{4\varsig\gamma}}}.
\end{align}
Now, plugging~\eqref{lemma-prob-upper-oracle-ranking} in~\eqref{equation-excess-risk}, we get
\begin{align*}
R(\widehat{\bcM}) - R(\ddt{\bcM}) &\leq \frac{\cst \rk(\ddt{\bcM})\rho^2\varsig^{-1}(\E[\norm{\bSigma_R}])^2}{(1/4e) + (1 - 1/\sqrt{4e})\sqrt{{3\rho}/{4\varsig\gamma}}} + 2\Lambda \sqrt{2\rk(\ddt{\bcM})}\norm{\widehat{\bcM} - \ddt{\bcM}}_F,
\end{align*}
where $\cst = 1024.$
Then using the fact that for any $a, b \in \R$, and $\epsilon > 0$, we have $2ab \leq a^2/(2\epsilon)  + 2\epsilon b^2$, we get for $\epsilon =p\varsig/4$
\begin{align*}
R(\widehat{\bcM}) - R(\ddt{\bcM}) &\leq \frac{\cst d_uDp^{-1}\rk(\ddt{\bcM})\rho^2\varsig^{-1}(\E[\norm{\bSigma_R}])^2}{(1/4e) + (1 - 1/\sqrt{4e})\sqrt{{3\rho}/{4\varsig\gamma}}}\\
&\quad + \Lambda^2d_uD(p\varsig/4)^{-1}\rk(\ddt{\bcM}) + \frac{p\varsig}{2d_uD}\norm{\widehat{\bcM} - \ddt{\bcM}}_F^2\\
&\leq \frac{\cst d_uDp^{-1}\rk(\ddt{\bcM})\rho^2\varsig^{-1}(\E[\norm{\bSigma_R}])^2}{(1/4e) + (1 - 1/\sqrt{4e})\sqrt{{3\rho}/{4\varsig\gamma}}}\\
&\quad + \Lambda^2d_uD(p\varsig/4)^{-1}\rk(\ddt{\bcM}) + \frac{\varsig}{2d_uD}\norm{\widehat{\bcM} - \ddt{\bcM}}_{\Pi,F}^2.
\end{align*}
Using Assumption~\ref{assumption-Bernstein-condition}, we obtain
\begin{align*}
R(\widehat{\bcM}) - R(\ddt{\bcM}) &\leq \frac{2\cst d_uDp^{-1}\rk(\ddt{\bcM})\rho^2\varsig^{-1}(\E[\norm{\bSigma_R}])^2}{(1/4e) + (1 - 1/\sqrt{4e})\sqrt{{3\rho}/{4\varsig\gamma}}}   + 8\Lambda^2d_uD(p\varsig)^{-1}\rk(\ddt{\bcM})\\
&\leq (p\varsig)^{-1}\rk(\ddt{\bcM})d_uD\Big(\frac{\rho^2(\E[\norm{\bSigma_R}])^2}{(1/4e) + (1 - 1/\sqrt{4e})\sqrt{{3\rho}/{4\varsig\gamma}}} + 8\Lambda^2 \Big).
\end{align*}
This finishes the proof of Theorem~\ref{theorem-oracle-ranking}.

\subsection{Proof of Lemma~\ref{lemma-upper-bound-bSigmastar}}
\label{proof-lemma-upper-bound-bSigmastar}

By the nonnegative factor and the sum properties of subdifferential calculus~\citep{boyd:2004:CO:993483}, we write 
\begin{equation*}
  \partial{R}_{\bcY}(\ddt{\bcM}) = \bigg\{\bcG = \frac{1}{d_uD}\sum_{v \in [V]}\sum_{(i,j)\in [d_u]\times [d_v]} B^v_{ij}G^v_{ij}E^v_{ij}: G^v_{ij}\in \partial\ell^v(Y^v_{ij}, \dt{M}^v_{ij})\bigg\}
\end{equation*}
Recall that the sudifferential of $\partial\ell^v(Y^v_{ij}, \dt{M}^v_{ij})$ at the point $\dt{M}^v_{ij}$ is defined as 
\begin{equation*}
  \partial\ell^v(Y^v_{ij}, \dt{M}^v_{ij}) = \{G^v_{ij}: \ell^v(Y^v_{ij}, {Q}^v_{ij}) \geq \ell^v(Y^v_{ij}, \dt{M}^v_{ij}) + G^v_{ij}({Q}^v_{ij} -\dt{M}^v_{ij})\}.
\end{equation*}
Thanks to Assumption~\ref{assumption-lipshitzloss}, we have, for all $G^v_{ij}\in \partial\ell^v(Y^v_{ij}, \dt{M}^v_{ij})$
\begin{equation*}
  |G^v_{ij}({Q}^v_{ij} -\dt{M}^v_{ij})| \leq |\ell^v(Y^v_{ij}, {Q}^v_{ij}) - \ell^v(Y^v_{ij}, \dt{M}^v_{ij})| \leq \rho_v |{Q}^v_{ij} - \dt{M}^v_{ij}|,
\end{equation*}
In particular, with ${Q}^v_{ij}\neq\dt{M}^v_{ij}$ for all $v\in[V]$ and $(i,j) \in [d_u]\times[d_v]$, we get $|G^v_{ij}| \leq \rho_v$. Then, any subgradient $\bcG$ of ${R}_{\bcY}$ has entries bounded by $\rho/(d_uD)$ (recall $\rho = \max_{v\in [V]} \rho_v$).
By a triangular inequality and the convexity of $\norm{\cdot}$, we have 
\begin{align*}
  \norm{\bcG} &\leq \norm{\bcG - \E[\bcG]} + \norm{\E[\bcG]}\\
              &\leq  \norm{\bcG - \E[\bcG]} + \E[\norm{\bcG}],
\end{align*}
for any subgradient $\bcG$ of ${R}_{\bcY}$.
On the one hand, we use the fact that $\E[\norm{\bcG}] \leq \E[\norm{\bcG}_F] \leq \sqrt{\E[\norm{\bcG}_F^2]}$.
Using~\eqref{upper-bound-marginal}, we have 
\begin{align*}
  \E[\norm{\bcG}_F^2] &\leq \frac 1{(d_uD)^2}\sum_{v\in[V]}\sum_{(i,j)\in[d_u]\times[d_v]}\rho_v^2\E[B_{ij}^v]\\
  &\leq \frac{\rho^2}{(d_uD)^2}\sum_{v\in[V]}\sum_{(i,j)\in[d_u]\times[d_v]}\pi^v_{ij}\\
  &\leq \frac{\rho^2\mu}{(d_uD)^2}.
\end{align*}
Now we apply Proposition~\ref{proposition-ctrl-spect-norm} to $\bcG - \E[\bcG]$.
Taking into account~\eqref{upper-bound-marginal}, we upper bound the constants $\kappa_1, \kappa_2$ and $\kappa_*$ as follows: 
\begin{align*}
\kappa_1  &= \frac{1}{d_uD}\max_{i \in[d_u]}\sqrt{\sum_{v\in[V]}\sum_{j\in[d_v]} \E[(B^v_{ij}G^v_{ij} - \E[B^v_{ij}G^v_{ij}								])^2]}\\
&\leq \frac{2\rho}{d_uD}\max_{i \in[d_u]}\sqrt{\sum_{v\in[V]}\sum_{j\in[d_v]} \pi^v_{ij}}\\
&\leq \frac{2\rho\sqrt{\mu}}{d_uD},
\end{align*}
\begin{align*}
\kappa_2  &= \frac{1}{d_uD}\max_{v \in[V]}\max_{j \in[d_v]}\sqrt{\sum_{i\in[d_u]} \E[(B^v_{ij}G^v_{ij} - \E[B^v_{ij}G^v_{ij}								])^2]}\\
&\leq \frac{2\rho}{d_uD}\max_{v \in[V]}\max_{j \in[d_v]}\sqrt{\sum_{i\in[d_u]} \pi^v_{ij}}\\
&\leq \frac{2\rho\sqrt{\mu}}{d_uD},
\end{align*}
and $\kappa_* = \frac{1}{d_uD}\max_{v\in[V]}\max_{(i,j)\in[d_u] \times[d_v]}|B^v_{ij}G^v_{ij} - \E[B^v_{ij}G^v_{ij}								]| \leq \frac{2\rho}{d_uD}$.
Now, choose $\epsilon = 1/2$ in Proposition~\ref{proposition-ctrl-spect-norm}, then we obtain, with probability at least $1 - 4(d_u \wedge D) e^{-x^2},$
\begin{align}
\label{inequltiybisgmastar}
\norm{\bcG - \E[\bcG]}
&\leq \frac{6\rho\sqrt{2\mu}+ 2\rho\sqrt{\cst_{1/2}}x}{d_uD}.
\end{align}
Setting $x = \sqrt{2\log(d_u +D)}$ in~\eqref{inequltiybisgmastar}, we get with probability at least $1 - 4/(d_u +D)$,
\begin{align}
\label{final-ctrl-bSigma1}
\norm{\bcG}
&\leq \frac{(1+6\sqrt{2})\rho\sqrt{\mu}+ 2\rho\sqrt{\cst_{1/2}}\sqrt{2\log(d_u+D)}}{d_uD},
\end{align}
for any subgradient $\bcG$ of ${R}_{\bcY}(\ddt{\bcM})$.

\section{Technical Lemmas}

In this section, we provide several technical lemmas, which are used for proving our main results.

\subsection{Useful lemmas for the proof of Theorem~\ref{theorem1}}
\label{appendix-useful-lemmas}

\begin{lemma}
\label{lemma-crtl-nucl-2}
Let $\bcA, \bcB \in \mathscr{C}_\infty(\gamma)$.
Assume that $\lambda \geq 2\norm{\nabla \mathscr{L}_{\bcY}(\bcB)}$, and $\mathscr{L}_{\bcY}(\bcA) + \lambda \norm{\bcA}_* \leq \mathscr{L}_{\bcY}(\bcB) + \lambda\norm{\bcB}_*.$ 
Then,
\begin{itemize}
\item [(i)] $\norm{\mathscr{P}^\perp_{\bcB}(\bcA - \bcB)}_ * \leq 3 \norm{\mathscr{P}_{\bcB}(\bcA - \bcB)}_*$,
\item [(ii)] $\norm{\bcA - \bcB}_* \leq 4\sqrt{2\rk(\bcB)}\norm{\bcA - \bcB}_F$.
\end{itemize}
\end{lemma}

\begin{proof}
We have  $\mathscr{L}_{\bcY}(\bcB)  - \mathscr{L}_{\bcY}(\bcA) \geq \lambda(\norm{\bcA}_* - \norm{\bcB}_*)$. 
\eqref{klopp-thm3} implies 
\begin{equation*}
\mathscr{L}_{\bcY}(\bcB)  - \mathscr{L}_{\bcY}(\bcA) \geq \lambda\big(\norm{\mathscr{P}^\perp_{\bcB}(\bcA - \bcB)}_* - \norm{\mathscr{P}_{\bcB}(\bcA - \bcB)}_*\big).
\end{equation*}
Moreover, by convexity of $\mathscr{L}_{\bcY}(\cdot)$ and the duality between $\norm{\cdot}_*$ and $\norm{\cdot}$ we obtain
\begin{align*}
\mathscr{L}_{\bcY}(\bcB)  - \mathscr{L}_{\bcY}(\bcA) &\leq \inr{\nabla\mathscr{L}_{\bcY}(\bcB),\bcB - \bcA}
\leq \norm{\nabla\mathscr{L}_{\bcY}(\bcB)} \norm{\bcB - \bcA}_*
\leq \frac{\lambda}{2}\norm{\bcB - \bcA}_*.
\end{align*}
Therefore,
\begin{equation}
\label{proof-lemma-1i}
\norm{\mathscr{P}^\perp_{\bcB}(\bcA - \bcB)}_* \leq \norm{\mathscr{P}_{\bcB}(\bcA - \bcB)}_* + \frac{1}{2}\norm{\bcA - \bcB}_*
\end{equation}
Using the triangle inequality, we get 
\begin{equation*}
\norm{\mathscr{P}^\perp_{\bcB}(\bcA - \bcB)}_* \leq 3\norm{\mathscr{P}_{\bcB}(\bcA - \bcB)}_*,
\end{equation*}
which proves $(i)$.
To prove $(ii)$, note that $\norm{\mathscr{P}_{\bcB}(\bcA)}_* \leq \sqrt{2\rk(\bcB)}\norm{\bcA}_F$, and $(i)$ imply
\begin{align*}
\norm{\bcA  - \bcB}_* 
& \leq 4 \sqrt{2\rk(\bcB)}\norm{\bcA - \bcB)}_F.
\end{align*}
\end{proof}

\begin{lemma}
\label{lemma-prob-upper-bound}
Let $\beta = \frac{946\gamma^2\log (d_u + D)}{pd_uD}$. Then, for all $\bcW \in \mathscr{C}(\beta, r)$,
\begin{equation*}
  \Big|\Delta^2_{\bcY}({\bcW}, \bcM) - (d_uD)^{-1}\norm{{\bcW} - \bcM}_{\Pi,F}^2]\Big| \leq \frac{(d_uD)^{-1}\norm{{\bcW} - \bcM}_{\Pi,F}^2}{2} + {1392r\gamma^2}(\E[\norm{\bSigma_R}])^2 + \frac{5567\gamma^2}{{d_uD}p}
\end{equation*}
with probability at least $1 - 4/(d_u+D)$.
\end{lemma}
\begin{proof}
We use a standard peeling argument. 
For any $\alpha > 1$ and $0 < \eta < 1 /2\alpha$, we define 
\begin{equation*}
\boldsymbol{\kappa} = \frac{1}{1/(2\alpha) - \eta}\Big(128\gamma^2r(\E[\norm{\bSigma_R}])^2 + \frac{512\gamma^2}{{d_uD}p}\Big)
\end{equation*}
and we consider the event
\begin{equation*}
\mathscr{W} = \bigg\{\exists\,\bcW \in \mathscr{C}(\beta, r): \Big|\Delta^2_{\bcY}({\bcW}, \bcM) - (d_uD)^{-1}\norm{{\bcW} - \bcM}_{\Pi,F}^2\Big| > \frac{(d_uD)^{-1}\norm{{\bcW} - \bcM}_{\Pi,F}^2}{2}  + \boldsymbol{\kappa} \bigg\}.
\end{equation*}
For $s \in \mathbb{N}^*$, set
\begin{equation*}
\mathcal{R}_s = \Big\{\bcW \in \mathscr{C}(\beta, r): \alpha^{s-1} \beta \leq (d_uD)^{-1}\norm{{\bcW} - \bcM}_{\Pi,F}^2 \leq \alpha^s\beta\Big\}. 
\end{equation*}
If the event $\mathscr{W}$ holds for some matrix $\bcW \in \mathscr{C}(\beta, r),$ then $\bcW$ belongs to some $\mathcal{R}_s$ and 
\begin{align*}
\Big|\Delta^2_{\bcY}({\bcW}, \bcM) - (d_uD)^{-1}\norm{{\bcW} - \bcM}_{\Pi,F}^2\Big| &\geq  \frac{(d_uD)^{-1}\norm{{\bcW} - \bcM}_{\Pi,F}^2}{2}  + \boldsymbol{\kappa}\\
  &\geq \frac{1}{2\alpha}\alpha^{s}\beta + \boldsymbol{\kappa}.
\end{align*}
For $\theta \geq \beta$ consider the following set of matrices
\begin{equation*}
\mathscr{C}(\beta, r, \theta) = \Big\{\bcW \in \mathscr{C}(\beta, r): (d_uD)^{-1}\norm{{\bcW} - \bcM}_{\Pi,F}^2 \leq \theta\Big\},
\end{equation*}
and the following event
\begin{equation*}
\mathscr{W}_s = \bigg\{\exists \, \bcW \in \mathscr{C}(\beta, r, \theta): \Big|\Delta^2_{\bcY}({\bcW}, \bcM) - (d_uD)^{-1}\norm{{\bcW} - \bcM}_{\Pi,F}^2\Big| \geq \frac{1}{2\alpha}\alpha^{s}\beta + \boldsymbol{\kappa}\bigg\}.
\end{equation*}
Note that $\bcW \in \mathscr{W}_s$ implies that $\bcW \in \mathscr{C}(\beta,r, \alpha^s\beta)$. 
Then, we get $\mathscr{W} \subset \cup_s \mathscr{W}_s$. 
Thus, it is enough to estimate the probability of the simpler event $\mathscr{W}_s$ and then apply a the union bound.
Such an estimation is given by the following lemma:
\begin{lemma}
\label{lemma-ctrl-processZUpperbd}
Let 
\begin{equation*}
\bcZ_\theta = \sup_{\bcW \in \mathscr{C}(\beta, r, \theta)} \Big|\Delta^2_{\bcY}({\bcW}, \bcM) - (d_uD)^{-1}\norm{{\bcW} - \bcM}_{\Pi,F}^2\Big|.
\end{equation*}
Then, we have  
\begin{equation*}
  \P\big[\bcZ_\theta > \frac{\theta}{2 \alpha}  +  \boldsymbol{\kappa}\big] \leq 4\exp\bigg(-\frac{pd_uD\eta^2\theta}{8\gamma^2}\bigg).
\end{equation*}
\end{lemma}
The proof of Lemma~\ref{lemma-ctrl-processZUpperbd} follows along the same lines of Lemma 10 in~\cite{klopp2015}.
We now apply an union bound argument combined to Lemma~\ref{lemma-ctrl-processZUpperbd}, we get 
\begin{align*}
\P[\mathscr{W}] \leq \P[\cup_{s=1}^\infty\mathscr{W}_s] &\leq 4\sum_{s=1}^\infty \exp\bigg(-\frac{pd_uD\eta^2\alpha^s\beta}{8\gamma^2}\bigg)\\
&\leq 4\sum_{s=1}^\infty \exp\bigg(-\frac{pd_uD\eta^2\beta\log \alpha}{8\gamma^2}s \bigg)\\
&\leq \frac{4\exp\bigg(-\frac{pd_uD\eta^2\beta\log \alpha}{8\gamma^2} \bigg)}{1 - \exp\bigg(-\frac{pd_uD\eta^2\beta\log \alpha}{8\gamma^2} \bigg)}.
\end{align*}
By choosing $\alpha = e, \eta= 1/4e$ and $\beta$ as stated we get the desired result.
\end{proof}

\subsection{Useful lemmas for the proof of Theorem~\ref{theorem-oracle-ranking}}
\label{appendix-useful-lemmas-theorem-orcale-ranking}

\begin{lemma}
\label{lemma-diff-Empir-loss}
Suppose $\Lambda \geq 2 \sup\{\norm{\bcG}: \bcG \in \partial{R}_{\bcY}(\ddt{\bcM})\}.$
Then
\begin{equation*}
\norm{\widehat{\bcM} - \ddt{\bcM}}_* \leq 4\sqrt{2\rk(\ddt{\bcM})}\norm{\widehat{\bcM}  - \ddt{\bcM}}_F.
\end{equation*}
\end{lemma}

\begin{proof}
For any subgradient $\bcG$ of ${R}_{\bcY}(\ddt{\bcM})$, we have ${R}_{\bcY}(\widehat{\bcM} ) \geq {R}_{\bcY}(\ddt{\bcM}) + \inr{\bcG, \widehat{\bcM}  - \ddt{\bcM}}.$
Then, the definition of the estimator $\widehat{\bcM}$, entails $R_{\bcY}(\ddt{\bcM}) - {R}_{\bcY}(\widehat{\bcM}) \geq \Lambda(\norm{\widehat{\bcM}}_* - \norm{\ddt{\bcM}}_*),$
hence $\inr{\bcG,\ddt{\bcM}-\widehat{\bcM}} \geq \Lambda(\norm{\widehat{\bcM}}_*-\norm{\ddt{\bcM}}_*).$
The duality between $\norm{\cdot}_*$ and $\norm{\cdot}$ yields 
\begin{equation*}
  \Lambda(\norm{\widehat{\bcM}}_*-\norm{\ddt{\bcM}_*}) \leq \norm{\bcG} \norm{\ddt{\bcM}-\widehat{\bcM}}_* \leq \frac \Lambda2 \norm{\ddt{\bcM}-\widehat{\bcM}}_*
\end{equation*}
then $\norm{\widehat{\bcM}}_*-\norm{\ddt{\bcM}_*} \leq \frac 12 \norm{\ddt{\bcM}-\widehat{\bcM}}_*.$
Now,~\eqref{klopp-thm3} implies 
\begin{equation*}
  \norm{\mathscr{P}^\perp_{\ddt{\bcM}}(\ddt{\bcM} - \widehat{\bcM})}_* \leq \norm{\mathscr{P}_{\ddt{\bcM}}(\ddt{\bcM} - \widehat{\bcM})}_* + \frac 12 \norm{\ddt{\bcM}-\widehat{\bcM}}_* \leq 3 \norm{\mathscr{P}_{\ddt{\bcM}}(\ddt{\bcM} - \widehat{\bcM})}_*.
\end{equation*}
Therefore $\norm{\ddt{\bcM} - \widehat{\bcM}}_* \leq 4 \norm{\mathscr{P}_{\ddt{\bcM}}(\ddt{\bcM} - \widehat{\bcM})}_*$.
Since $\norm{\mathscr{P}_{\ddt{\bcM}}(\ddt{\bcM} - \widehat{\bcM})}_* \leq \sqrt{2\rk(\ddt{\bcM})} \norm{\ddt{\bcM} - \widehat{\bcM}}_F$, we establish the proof of Lemma~\ref{lemma-diff-Empir-loss}. 
\end{proof}

\begin{lemma}
\label{lemma-prob-upper-bound-loss-case}
Let 
\begin{equation*}
  \nu = \frac{32\big(1 + e\sqrt{{3\rho}/{\varsig \gamma}}\big)\rho\gamma\log(d_u+D)}{3pd_uD},
 \end{equation*} 
then, with probability at least $1 - 4/(d_u+D)$, the following holds uniformly over $\bcQ \in \mathscr{Q}(\nu, r)$
\begin{align*}
  \Big|\big(R_{\bcY}({\bcQ})& - {R}_{\bcY}(\ddt{\bcM})\big) - \big(R({\bcQ}) - R(\ddt{\bcM})\big)\Big|\\
  &\leq \frac{R({\bcQ}) - R(\ddt{\bcM})}{2} + \frac{16}{(1/4e) + (1 - 1/\sqrt{4e})\sqrt{{3\rho}/{4\varsig\gamma}}}{r\rho^2(p\varsig)^{-1}}(\E[\norm{\bSigma_R}])^2.
\end{align*}
\end{lemma}
\begin{proof}
The proof is based on the peeling argument.
For any $\delta > 1$ and $0 < \vartheta < 1 /2\delta$, define 
\begin{equation}
\label{boldzeta-oracle-ranking}
\boldsymbol{\zeta} = \frac{16{r(p\varsig)^{-1}}\rho^2(\E[\norm{\bSigma_R}])^2}{({1}/{2\delta}) + \sqrt{{3\rho}/{4\varsig \gamma}} - \Big(\vartheta +\sqrt{{3\rho}/{4\varsig\gamma}\vartheta}\Big)},
\end{equation}
and we consider the event
\begin{equation*}
\mathscr{A} = \bigg\{\exists\,\bcQ \in \mathscr{Q}(\nu, r): \Big|\big(R_{\bcY}({\bcQ}) - {R}_{\bcY}(\ddt{\bcM})\big) - \big(R({\bcQ}) - R(\ddt{\bcM})\big)\Big| > \frac{R({\bcQ}) - R(\ddt{\bcM})}{2}  + \boldsymbol{\zeta} \bigg\}.
\end{equation*}
For $ l\in \mathbb{N}^*$, we define the sequence of subsets 
\begin{equation*}
\mathcal{J}_l = \Big\{\bcQ \in \mathscr{Q}(\nu, r): \delta^{l-1} \nu \leq R({\bcQ}) - R(\ddt{\bcM}) \leq \delta^l\nu\Big\}. 
\end{equation*}
If the event $\mathscr{A}$ holds for some matrix $\bcQ \in \mathscr{Q}(\nu, r),$ then $\bcQ$ belongs to some $\mathcal{J}_l$ and 
\begin{align*}
\Big|\big(R_{\bcY}({\bcQ}) - {R}_{\bcY}(\ddt{\bcM})\big) - \big(R({\bcQ}) - R(\ddt{\bcM})\big)\Big| &> \frac{R({\bcQ}) - R(\ddt{\bcM})}{2}  + \boldsymbol{\zeta}\\
  &\geq \frac{1}{2\delta}\delta^{l}\nu + \boldsymbol{\zeta}.
\end{align*}
For $\theta \geq \nu$, consider the following set of matrices
\begin{equation*}
\mathscr{Q}(\nu, r, \theta) = \Big\{\bcQ \in \mathscr{Q}(\nu, r): R(\bcQ) - R(\ddt{\bcM}) \leq \theta\Big\},
\end{equation*}
and the following event
\begin{equation*}
\mathscr{A}_l = \bigg\{\exists \, \bcQ \in \mathscr{Q}(\nu, r, \theta):\Big|\big(R_{\bcY}({\bcQ}) - {R}_{\bcY}(\ddt{\bcM})\big) - \big(R({\bcQ}) - R(\ddt{\bcM})\big)\Big| \geq \frac{1}{2\delta}\delta^{l}\nu + \boldsymbol{\zeta}\bigg\}.
\end{equation*}
Note that $\bcQ \in \mathcal{J}_l$ implies that $\bcQ \in \mathscr{Q}(\nu,r, \delta^l\nu)$. 
Then, we get $\mathscr{A} \subset \cup_{l} \mathscr{A}_l$. 
Thus, it is enough to estimate the probability of the simpler event $\mathscr{A}_l$ and then apply a the union bound.
Such an estimation is given in Lemma~\ref{supermum-process-inLosslip}, where we derive a concentration inequality for the following supremum of process:
\begin{align*}
{\bXi}_\theta = \sup_{\bcQ \in \mathscr{Q}(\nu, r, \theta)}\Big| \big(R_{\bcY}(\bcQ) - {R}_{\bcY}(\ddt{\bcM})\big) - \big(R(\bcQ) - R(\ddt{\bcM})\big)\Big|
\end{align*}
We now apply an union bound argument combined to Lemma~\ref{supermum-process-inLosslip}, we get 
\begin{align*}
\P[\mathscr{A}] \leq \P[\cup_{l=1}^\infty\mathscr{A}_l] &\leq \sum_{l=1}^\infty \exp\Big(-\frac{3d_uD\vartheta\delta^l\nu}{8\rho\gamma}\Big)\\
&\leq \sum_{l=1}^\infty \exp\Big(-\frac{3d_uD\vartheta\log(\delta)\nu}{8\rho\gamma}l\Big)\\
&\leq \frac{\exp\Big(-\frac{3d_uD\vartheta\log(\delta)\nu}{8\rho\gamma}\Big)}{1 - \exp\Big(-\frac{3d_uD\vartheta\log(\delta)\nu}{8\rho\gamma}\Big)},
\end{align*}
where se used the elementary inequality that $u^s =e^{s\log(u)} \geq s\log(u)$.
By choosing $\delta = e, \vartheta= 1/4e$ and $\nu$ as stated we get the desired result.
\end{proof}

\begin{lemma}
\label{supermum-process-inLosslip}
One has
\begin{equation*}
\P\Big[\bXi_\theta \geq \Big(1 + \delta\sqrt{\frac{3\rho}{\varsig \gamma}}\Big)\frac{\theta}{2\delta} + \boldsymbol{\zeta}\Big] \leq \exp\Big(-\frac{3d_uD\vartheta\theta}{8\rho\gamma}\Big).
\end{equation*}
\end{lemma}
\begin{proof}
The proof of this lemma is based on Bousquet's concentration theorem:
\begin{theorem}\citep{bousquet2002} (see also Corollary 16.1 in~\cite{van2016estimation})
\label{bousquet-theorem}
Let $\mathcal{F}$ be a class of real-valued functions.
Let $T_1, \ldots, T_N$ be independent random variables such that $\E[f(T_i)] =0$ and $|f(T_i)|\leq \xi$ for all $i=1, \ldots, N$ and for all $f \in \mathcal{F}.$
Introduce 
$Z = \sup_{f \in \mathcal{F}}\Big| \frac{1}{N}\sum_{i=1}^N\big(f(T_i) - \E[f(T_i)]\big)\Big|$.
Assume further that 
\begin{equation*}
\frac{1}{N}\sum_{i=1}^N \sup_{f \in \mathcal{F}} \E\big[f^2(T_i)\big] \leq M^2.
\end{equation*}
Then we have for all $t > 0$
\begin{equation*}
\P\bigg[Z \geq 2 \E[Z] + M \sqrt{\frac{2t}{N}} + \frac{4t\xi}{3N}\bigg] \leq e^{-t}.
\end{equation*}
\end{theorem}
We start by bounding the expectation 
\begin{align*}
\E[\bXi_\theta] &= \E\Big[\sup_{\bcQ \in \mathscr{Q}(\nu, r, \theta)}\Big| \big(R_{\bcY}(\bcQ) - {R}_{\bcY}(\ddt{\bcM})\big) - \big(R(\bcQ) - R(\ddt{\bcM})\big)\Big|\Big]\\
&= \E\Big[\sup_{\bcQ \in \mathscr{Q}(\nu, r, \theta)}\Big| \big(R_{\bcY}(\bcQ) - {R}_{\bcY}(\ddt{\bcM})\big) - \E\big[R_{\bcY}(\bcQ) - {R}_{\bcY}(\ddt{\bcM})\big]\Big|\Big]\\
&= \E\Big[\sup_{\bcQ \in \mathscr{Q}(\nu, r, \theta)}\Big| 
\frac{1}{d_uD}\sum_{v\in [V]}\sum_{(i,j)\in[d_u]\times[d_v]}B^v_{ij}\big(\ell^v(Y^v_{ij},Q^v_{ij}) - \ell^v(Y^v_{ij},\dt{M}^v_{ij})\big)\\
&\hspace{7cm} - \E\big[B^v_{ij}\big(\ell^v(Y^v_{ij},Q^v_{ij}) - \ell^v(Y^v_{ij},\dt{M}^v_{ij})\big)\big]\Big|\Big]\\
&\leq 2\E\Big[\sup_{\bcQ \in \mathscr{Q}(\nu, r, \theta)}\Big|\frac{1}{d_uD}\sum_{v\in [V]}\sum_{(i,j)\in[d_u]\times[d_v]}\varepsilon_{ij}^vB^v_{ij}\big(\ell^v(Y^v_{ij},Q^v_{ij}) - \ell^v(Y^v_{ij},\dt{M}^v_{ij})\big)\Big|\Big]\\
&\leq 4\rho\E\Big[\sup_{\bcQ \in \mathscr{Q}(\nu, r, \theta)}\Big|\frac{1}{d_uD}\sum_{v\in [V]}\sum_{(i,j)\in[d_u]\times[d_v]}\varepsilon_{ij}^vB^v_{ij}(Q^v_{ij} - \dt{M}^v_{ij})\Big|\Big]\\
&\leq 4\rho\E\Big[\sup_{\bcQ \in \mathscr{Q}(\nu, r, \theta)}\Big|\inr{\bSigma_R, \bcQ - \ddt{\bcM}}\Big|\Big]\\
&\leq 4\rho\E\Big[\norm{\bSigma_R}\sup_{\bcQ \in \mathscr{Q}(\nu, r, \theta)}\norm{\bcQ - \ddt{\bcM}}_*\Big],
\end{align*}
where the first inequality follows from symmetrization of expectations theorem of van der Vaart and Wellener, the second from contraction principle of Ledoux and Talagrand (see Theorems 14.3 and 14.4 in~\cite{buhlmann2011statistics}), and the third from duality between nuclear and operator norms.
We have $\bcQ \in \mathscr{Q}(\nu, r, \theta)$ then $\norm{\bcQ - \ddt{\bcM}}_* \leq \sqrt{r}\norm{{\bcQ} - \ddt{\bcM}}_{F}$ and using Assumption~\ref{assumption-Bernstein-condition}, we have $\norm{\bcQ - \ddt{\bcM}}_* \leq \sqrt{r(p\varsig)^{-1}\big(R(\bcQ) - R(\ddt{\bcM})\big)} \leq \sqrt{r(p\varsig)^{-1}\theta}.$
Then,
\begin{equation*}
\E[\bXi_\theta] \leq 4\sqrt{r(p\varsig)^{-1}\theta}\rho\E[\norm{\bSigma_R}].
\end{equation*}
For the upper bound $\xi$ in~Theorem~\ref{bousquet-theorem}, we have that
\begin{equation*}
\big|\ell^v(Y^v_{ij},Q^v_{ij}) - \ell^v(Y^v_{ij},\dt{M}^v_{ij})\big| \leq \rho_v|Q^v_{ij} - \dt{M}^v_{ij}\big| \leq 2\rho_v\gamma \leq 2\rho\gamma.
\end{equation*}
Now we compute $M$ in~Theorem~\ref{bousquet-theorem}. Thanks to Assumption~\ref{assumption-Bernstein-condition}, we have  
\begin{align*}
\frac{1}{d_uD}\sum_{v\in [V]}&\sum_{(i,j)\in[d_u]\times[d_v]}\E\big[\big(B^v_{ij}\big(\ell^v(Y^v_{ij},Q^v_{ij}) - \ell^v(Y^v_{ij},\dt{M}^v_{ij})\big)\big)^2\big]\\
&\leq \frac{1}{d_uD}\sum_{v\in [V]}\sum_{(i,j)\in[d_u]\times[d_v]} (\rho_v)^2\E\big[B^v_{ij}(Q^v_{ij} - \dt{M}^v_{ij})^2]\\
&\leq \frac{\rho^2}{d_uD}\norm{\bcQ - \ddt{\bcM}}_{\Pi, F}^2\\%
&\leq \frac{\rho^2}{\varsig}(R(\bcQ) - R(\ddt{\bcM}))\\
&\leq \frac{\rho^2\theta}{\varsig}.
\end{align*}
Then, Bousquet's theorem implies that for all $t>0$,
\begin{equation*}
\P\Big[\bXi_\theta \geq 2\E[\bXi_\theta] + \sqrt{\frac{2\rho^2\theta t}{\varsig d_uD}} + \frac{8\rho \gamma t}{3d_uD}\Big]\leq e^{-t}.
\end{equation*}
Taking $t = \frac{3d_uD\vartheta\theta}{8\rho\gamma}$, we obtain
\begin{equation}
\label{equation-bousquet-bXi}
\P\Big[\bXi_\theta \geq 8\gamma\sqrt{r(p\varsig)^{-1}\theta}\rho\E[\norm{\bSigma_R}] + \Big(\sqrt{\frac{3\rho}{4\varsig\gamma}\vartheta} + \vartheta\Big)\theta\Big]\leq \exp\Big(-\frac{3d_uD\vartheta\theta}{8\rho\gamma}\Big).
\end{equation}
Using the fact that for any $a,b\in\R$, and $\epsilon>0$, $2ab \leq a^2/\epsilon + \epsilon b^2$, we get (for $\epsilon = {1}/{2\delta} + \sqrt{{3\rho}/{4\varsig \gamma}} - \Big(\vartheta +\sqrt{{3\rho\vartheta}/{4\varsig\gamma}}\Big)$), we get 
\begin{align*}
8\gamma\sqrt{r(p\varsig)^{-1}\theta}\rho\E[\norm{\bSigma_R}] + \Big(\sqrt{\frac{3\vartheta\rho}{\varsig\gamma}}+ \vartheta\Big)\theta
&\leq\frac{16{r(p\varsig)^{-1}}\rho^2(\E[\norm{\bSigma_R}])^2}{\frac{1}{2\delta} + \sqrt{\frac{3\rho}{4\varsig \gamma}} - \vartheta -\sqrt{\frac{3\rho}{4\varsig\gamma}\vartheta}} + \Big(\frac{1}{2\delta} + \sqrt{\frac{3\rho}{4\varsig \gamma}}\Big)\theta\\
&\leq\frac{16{r(p\varsig)^{-1}}\rho^2(\E[\norm{\bSigma_R}])^2}{\frac{1}{2\delta} + \sqrt{\frac{3\rho}{4\varsig \gamma}} - \vartheta -\sqrt{\frac{3\rho}{4\varsig\gamma}\vartheta}} + \Big(1 + \delta\sqrt{\frac{3\rho}{\varsig \gamma}}\Big)\frac{\theta}{2\delta}.
\end{align*}

Using~\eqref{equation-bousquet-bXi}, we get $\P\Big[\bXi_\theta \geq \Big(1 + \delta\sqrt{\frac{3\rho}{\varsig \gamma}}\Big)\frac{\theta}{2\delta} + \boldsymbol{\zeta}\Big] \leq \exp\Big(-\frac{3d_uD\vartheta\theta}{8\rho\gamma}\Big)$.
This finishes the proof of Lemma~\ref{supermum-process-inLosslip}.
\end{proof}

\section{Sub-exponential random variables}
\label{appendix-sub-exponentail-RV}

The material here is taken from~\cite{vershyni2010}.
\begin{definition}
A random variable $X$ is sub-exponential with parameters $(\omega, b)$ if for all $t$ such that $|t| \leq 1/b$,
\begin{equation}
\label{tail-direct-from-Chernoff-sub-expo}
\E\big[\exp\big(t(X - \E[X])\big)\big] \leq \exp\big(\frac{t^2 \omega^2}{2}\big).
\end{equation}
\end{definition}
When $b=0$, we interpret $1/0$ as being the same as $\infty$, it follows immediately from this definition that any sub-Gaussian random variable is also sub-exponential.
There are also a variety of other conditions equivalent to sub-exponentiality, which we relate by defining the sub-exponential norm of random variable. 
In particular, we define the sub-exponential norm (sometimes known as the $\psi_1$-Orlicz in the literature) as 
\begin{equation*}
\norm{X}_{\psi_1} := \sup_{q \geq 1}\frac{1}{q}(\E[|X^q|])^{1/q}.
\end{equation*}
Then we have the following lemma which provides several equivalent characterizations of sub-exponential random variables.
\begin{theorem}
\label{lemma-properties-sub}
(Equivalence of sub-exponential properties~\citep{vershyni2010})\\
Let $X$ be a random variable and $\omega > 0$ be a constant. 
Then, the following properties are all equivalent with suitable numerical constants $K_i > 0, i=1, \ldots, 4$, that are different from each other by at most an absolute constant $\cst$, meaning  that if one statement $(i)$ holds with parameter $K_i$, then the statement $(j)$ holds with parameter $K_j \leq \cst K_i$.
\begin{enumerate}
\item [(1)] sub-exponential tails: $\P[|X| > t] \leq \exp\big(1 -\frac{t}{\omega K_1}\big)$, for all $t \geq 0$.
\item [(2)] sub-exponential moments: $(\E[|X^q|])^{1/q} \leq K_2\omega {q},$ for all $q\geq 1$.
\item [(3)] existence of moment generating function (Mgf): $\E\big[\exp\big(\frac{X}{\omega K_3}\big)\big] \leq e.$
\end{enumerate}
\end{theorem}
Note that in each of the statements of Theorem~\ref{lemma-properties-sub}, we may replace $\omega$ by $\norm{X}_{\psi_1}$ and, up to absolute constant factors, $\norm{X}_{\psi_1}$ is the smallest possible number in these inequalities.

\begin{lemma} (Mgf of sub-exponential random variables~\citep{vershyni2010}) 
\label{lemma-mgf-subexprv}
Let $X$ be a centered sub- exponential random variable.
Then, for $t$ such that $|t| \leq c/\norm{X}_{\psi_1},$ one has 
\begin{equation*}
\E[\exp(tX)] \leq \exp(C t^2 \norm{X}^2_{\psi_1})
\end{equation*}
where $C, c > 0$ are absolute constants.
\end{lemma}
\begin{lemma}
\label{lemma-subgaussian-tail-X}
For all $v\in[V]$ and $(i,j)\in [d_u]\times [d_v]$, the random variable $X^v_{i,j}$ is a sub-exponential with parameters $(U_\gamma, K)$, where $K$ is defined in Assumption~\ref{assump-Gv-bound}.
Moreover, we have that $\norm{X^v_{i,j}}_{\psi_1} = \cst (U_\gamma \vee K)$ for some absolute constant $\cst$.
\end{lemma}
\begin{proof}
Let $t$ such that $|t|\leq 1/K$, then 
\begin{align*}
\E[\exp\big(t(X^v_{ij} &- \E[X^v_{ij}])\big)]\\
&= e^{-t(G^v)'(M^v_{ij})}\int_{\R} h^v(x)\exp\big((t + M^v_{ij})x - G^v(M^v_{ij})\big)dx\\
& = e^{G^v(t+M_{ij}^v) - G^v(M_{ij}^v) -t(G^v)'(M^v_{ij})}\int_{\R} h^v(x)\exp\big((t + M^v_{ij})x - G^v(t + M^v_{ij})\big)dx\\
& = e^{G^v(t+M_{ij}^v) - G^v(M_{ij}^v) -t(G^v)'(M^v_{ij})},
\end{align*}
where we used in the last inequality the fact that that $\int_{\R} h^v(x)\exp\big((t + M^v_{ij})x - G^v(t + M^v_{ij})\big)dx = \int_\R f_{h^v,G^v}(X^v_{i,j}|t+M^v_{ij})dx = 1$. 
Therefore, an ordinary Taylor series expansion of $G^v$ implies that there exists $t_{\gamma, K} \in [-\gamma - \frac{1}{K}, \gamma + \frac{1}{K}]$ such that $G^v(t+M_{ij}^v) - G^v(M_{ij}^v) -t(G^v)'(M^v_{ij}) = (t^2/2)(G^v)''(t_{\gamma, K}^2)$.
By Assumption~\ref{assump-Gv-bound}, we obtain 
\begin{equation*}
\E[\exp\big(t(X^v_{ij} - \E[X^v_{ij}])\big)] \leq \exp\Big(\frac{t^2U_\gamma^2}{2}\Big).
\end{equation*}
Using Lemma~\ref{lemma-mgf-subexprv}, we get $\norm{X^v_{i,j}}_{\psi_1} = \cst (U_\gamma \vee K)$ for some absolute constant $\cst$.
This proves Lemma~\ref{lemma-subgaussian-tail-X}.
\end{proof}

\end{appendices}

\bibliography{biblio}
\bibliographystyle{chicago}

\end{document}